\pdfoutput=1
\documentclass[11pt]{article}

\usepackage{times}
\usepackage{latexsym}
\usepackage[numbers]{natbib}

\usepackage[T1]{fontenc}
\usepackage[utf8]{inputenc}

\usepackage{microtype}

\usepackage{inconsolata}

\usepackage{graphicx}

\usepackage{amsmath,amsfonts,bm}

\def\1{\bm{1}}

\DeclareMathAlphabet{\mathsfit}{\encodingdefault}{\sfdefault}{m}{sl}
\SetMathAlphabet{\mathsfit}{bold}{\encodingdefault}{\sfdefault}{bx}{n}

\newcommand{\E}{\mathbb{E}}

\newcommand{\R}{\mathbb{R}}

\DeclareMathOperator*{\argmin}{arg\,min}

\usepackage{hyperref}
\usepackage{url}
\usepackage{graphicx}
\usepackage{xspace}

\usepackage{amsmath,bm,amssymb,amsthm}

\newcommand{\LSP}{\textup{LSP}}
\newcommand{\algo}{NGDiff\xspace}
\newcommand{\g}{\bm{g}}
\renewcommand{\d}{\bm{d}}

\renewcommand{\H}{\mathbf{H}}
\renewcommand{\R}{\textup{R}}
\newcommand{\F}{\textup{F}}
\newcommand{\G}{\mathbf{G}}

\usepackage{times,paralist}
\usepackage{enumitem}
\usepackage{graphicx}
\usepackage{floatrow}
\usepackage{subfig}
\usepackage[normalem]{ulem}
\usepackage{makecell}
\usepackage{amsmath}
\usepackage{multirow}
\usepackage{amssymb}
\usepackage{pgffor}
\usepackage{bbding}
\usepackage{wrapfig}
\usepackage{cleveref}
\usepackage[noend]{algpseudocode}
\usepackage{algorithm}

\usepackage{cleveref}
\usepackage{booktabs} % To thicken table lines
\usepackage[tableposition=top]{caption}

\graphicspath{{figs/}}

\makeatletter
\newcommand*\bigcdot{\mathpalette\bigcdot@{2}}
\newcommand*\bigcdot@[2]{\mathbin{\vcenter{\hbox{\scalebox{#2}{$\m@th#1\bullet$}}}}}
\makeatother

\newtheorem{theorem}{Theorem}
\newtheorem{othertheorem}{othertheorem}[section]
\newtheorem{lemma}[theorem]{Lemma}

\newtheorem{remark}[othertheorem]{Remark}

\newtheorem{definition}[theorem]{Definition}

\title{Unlearning as multi-task optimization: A normalized gradient difference approach with an adaptive learning rate}

\author{Zhiqi Bu\thanks{Equal contribution. Work done during Xiaomeng Jin's internship at Amazon. Corresponding author: \texttt{zhiqibu@amazon.com}} \\
  Amazon AGI\\
  \and
  Xiaomeng Jin$^*$ \\
UIUC\\
\and
  Bhanukiran Vinzamuri \\
Amazon AGI\\
\and
  Anil Ramakrishna \\
Amazon AGI\\
\and
  Kai-Wei Chang\thanks{Concurrent positions as an Amazon Scholar and as a faculty at the corresponding institutes. This paper represents the work performed at Amazon.} \\
Amazon AGI \& UCLA\\
\and
  Volkan Cevher$^\dagger$ \\
Amazon AGI \& LIONS EPFL\\
\and
  Mingyi Hong$^\dagger$ \\
Amazon AGI \& University of Minnesota\\
}
\date{}

\usepackage[verbose=true,letterpaper]{geometry} 
\usepackage{natbib}
\usepackage{xcolor}

\begin{document}
\maketitle

\begin{abstract}

Machine unlearning has been used to remove unwanted knowledge acquired by large language models (LLMs). In this paper, we examine machine unlearning from an optimization perspective, framing it as a \emph{regularized multi-task optimization problem}, where one task optimizes a forgetting objective and another optimizes the model performance. In particular, we introduce a normalized gradient difference (\algo) algorithm, enabling us to have better control over the trade-off between the objectives, while integrating a new, automatic learning rate scheduler. We provide a theoretical analysis and empirically demonstrate the superior performance of \algo among state-of-the-art unlearning methods on the \textit{TOFU} and \textit{MUSE} datasets while exhibiting stable training. 
\end{abstract}

\section{Introduction}
%% Optional introductory paragraph
%% KC: I added some intro back to on board the NLP audience

Large language models (LLMs) consume a large amount of data during pre-training. %, from 1T tokens for Falcon \cite{refinedweb}, OLMo \cite{groeneveld2024olmoacceleratingsciencelanguage} to 15T tokens for Llama3 model. 
After the model is built, we may have to unlearn certain data points that contain potentially sensitive, harmful, or copyrighted content. As re-training from scratch in such a case is not feasible due to the associated costs,  researchers have developed a number of machine unlearning methods applied after training.

\if 0
%% previous 
Large language models (LLMs) are trained on a huge collection of data from various sources, such as books and websites. For example, LLAMA3 is pre-trained on over 15T tokens \cite{dubey2024llama}, while Falcon \cite{refinedweb}, OLMo \cite{groeneveld2024olmoacceleratingsciencelanguage} are each pre-trained on over 1T tokens. However, this extensive data pre-training raises serious concerns about data risks for the following reasons: (1) Certain data sources, despite being publicly available, contain potentially harmful or sensitive content, such as nudity, personal information, copyright-protected information, etc. (2) LLMs can memorize these data during training, and then re-generate them.  For example, Exhibit J from The New York Times \cite{exhibitj} shows examples of outputs from GPT-4 that ``contain large spans that are identical to the actual text of the article from The New York Times,'' which are copyrighted; some attacks can extract an individual’s name, email address, phone number, and physical address from LLMs \cite{carlini2021extracting}. Furthermore, the risk of memorization increases with the size of the model (see Figure 1 and 3 in \cite{carliniquantifying}), exposing bigger models with higher utility to greater risks of data memorization.
w
\fi 

% To tackle such risks, one naive but computationally expensive approach is to retrain the models without the problematic data, which can be infeasible given that LLMs cost millions of dollars and take months to train, and the cost escalates each time new data are identified as problematic. 

\if 
%% previous version
Although retraining the models by removing the problematic data can resolve this issue, this approach is not feasible given that LLMs cost millions of dollars and take months to train, and the cost escalates each time new data are identified as problematic. Recently, researchers have developed a number of machine unlearning methods, which are applied after the models have completed the training and memorized the data. Specifically, we divide the data into two classes: the retaining set (\R), on which the model can memorize and have high performance, and the forgetting set (\F), which the model should not memorize. The goal of unlearning is to continue training the model in a way, so that knowledge from the forgetting set is effectively removed, and that the unlearned model behaves similarly to one that is retrained solely on the retaining set.

%% VC: Let's move this to the experimental section.
\begin{figure}[t]
\centering
\includegraphics[width=0.95\linewidth]{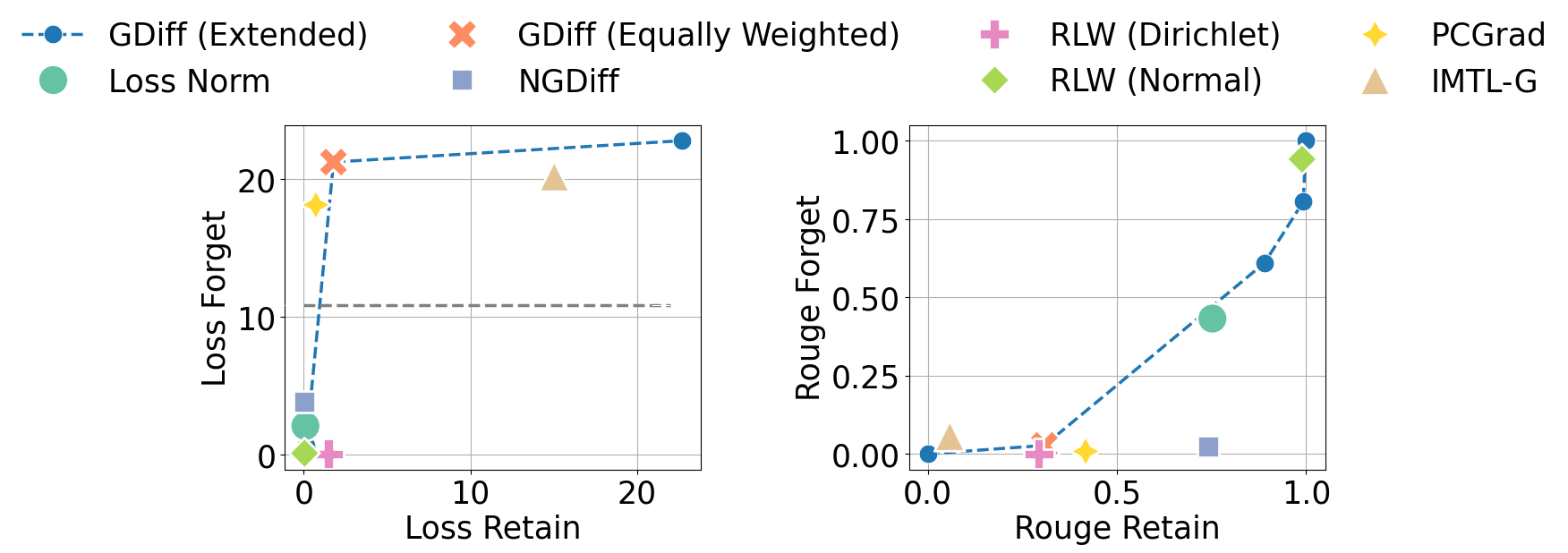}
    \caption{Loss values and ROUGE scores on the forgetting and retaining data from the \textit{TOFU} dataset using different unlearning methods on the Phi-1.5 language model. We apply the \textit{extended GDiff} with various coefficients (see \eqref{eq:static}, $0\leq c\leq 1$) and connect the results with a blue dashed line. We denote MTO methods as different markers, and use a grey dashed line to represent the loss of random guess.}
    % the $\textcolor{teal}{\bigcdot}$ marker represents method using loss normalization method, and the \textcolor{orange}{\XSolidBold{}} marker represents the result using gradient normalization method.}
    \label{fig:intro-example}
\end{figure}
\fi 
% Because of the strong correlation between loss values and the extent of memorization, existing machine unlearning methods are primarily formulated as optimization methods,
% \cite{liu2023muter, chen2024machine, liu2024rethinking}. For instance, Gradient Ascent (GA) aims to maximize the training loss on the forgetting data. However, this update direction can also negatively impact the model's performance on the retaining data. On the opposite, Gradient Descent (GD) aims to continue minimizing the training loss on the retaining data, hoping that the phenomenon of catastrophic forgetting will remove the unwanted data, which may be practically ineffective. In the middle ground, Gradient Difference (GDiff) combines both GA on the forgetting data and GD on the retaining data in each iteration, which has been empirically effective \cite{liu2022continual,maini2024tofu}. Negative Preference Optimization (NPO) \cite{zhang2024negative} further addresses this issue by assigning a low likelihood to the forgetting data, thereby balancing the unlearning of the forgetting data while preserving the model's utilities. 

\looseness=-1Existing machine unlearning methods are formulated primarily as minimizing memorization through the language model loss~\cite{jang-etal-2023-knowledge, chen2024machine, liu2024rethinking}. 
%Specifically, we divide the data into two classes: the retaining set (\R), on which the model can memorize and have high performance, and the forget set (\F), which the model should not memorize. The optimization objective of unlearning is selected in a way, so that knowledge from the forgetting set is effectively removed, and that the unlearned model behaves similarly to one that is retrained solely on the retain set. 
In particular, the Gradient Ascent (GA) method maximizes the language model (LM) loss (i.e., minimizes the negative LM loss) on the target forget set (\F). However, this approach can also negatively affect the utility of the model. To mitigate the utility loss, the Gradient Difference (GDiff) method selects a subset of the training data as the retain set (\R), minimizing the sum of the negative LM loss on the forgetting set and the standard LM loss on the retaining set. This approach has been empirically shown to effectively preserve the model's performance~\cite{liu2022continual,maini2024tofu}. Similarly, Negative Preference Optimization (NPO)~\cite{zhang2024negative} assigns a lower likelihood of forgetting data, thereby balancing the unlearning performance with model utility. %SCRUB here? 

Despite these successes, there are still two key issues preventing the methods from reaching their full potential. First, balancing retaining and forgetting losses is difficult (\Cref{fig:intro-example}) given the disproportionate sizes of the forget and retain datasets. Second, the optimization methods for unlearning are usually sensitive to the learning rate (\textit{cf.}, \Cref{sec:preliminary_evidence}, \Cref{fig:grad_diff_lr_rouge}). For instance, various learning rates can lead to substantial changes in the ROUGE scores and loss values even for the same algorithm, making the unlearning methods unstable and difficult to use in practice.

In this paper, we carefully examine unlearning from an optimization perspective and formulate it as a multi-task optimization (MTO) problem \cite{chen2021multi, xin2022current}: we aim to minimize the LM loss (i.e., maximize the utility) on the retaining set and maximize the LM loss on the forgetting set (i.e., minimize memorization), simultaneously.\footnote{A naive approach is optimizing the sum of these two objectives. We will discuss alternatives to improve upon this. } To solve this two-task problem, we study the rich literature of multi-task methods that seeks the Pareto optimality of two tasks (e.g., IMTL \cite{liu2021towards}, GradNorm \cite{chen2018gradnorm}, RLW \cite{linreasonable}, PCGrad \cite{yu2020gradient}, and scalarization \cite{boyd2004convex}), and design an approach specifically for the LLM unlearning problem.

% \begin{figure}[t]
% \centering
% % \includegraphics[width=0.45\linewidth]{figures/all_method_rouge_forget_retain_last_step.png}
% \includegraphics[width=0.9\linewidth]{figures/rouge_final.png}
%     \caption{Memorization (measured by ROUGE) on \textit{TOFU} forgetting and retaining sets with Phi-1.5 (see setup in Sec \ref{sec:exp}). The dashed line represents \textit{extended GDiff} with $0 \leq c \leq 1$ in \eqref{eq:static}. Our  \algo (gray square) achieves the best trade-off (bottom right preferred) among other unlearning methods.}
%     % the $\textcolor{teal}{\bigcdot}$ marker represents method using loss normalization method, and the \textcolor{orange}{\XSolidBold{}} marker represents the result using gradient normalization method.}
%     \label{fig:intro}
% \end{figure}
\begin{figure}[!htb]
\centering
\includegraphics[width=0.95\linewidth]{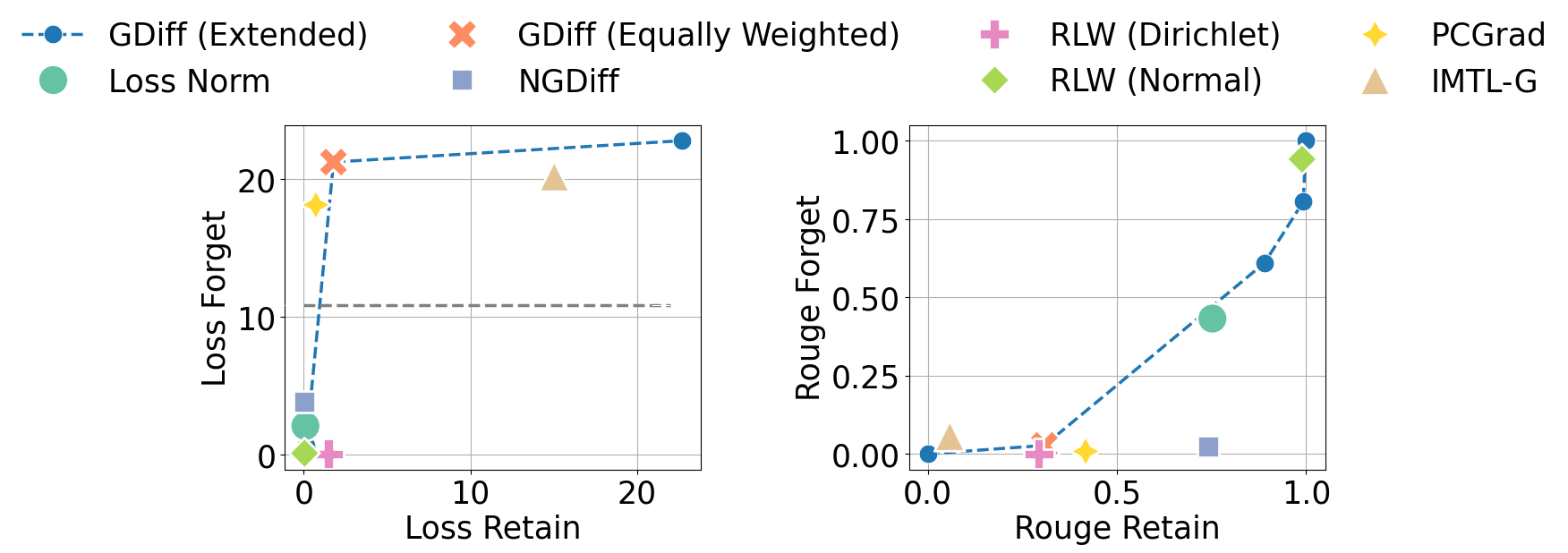}
    \caption{Loss values and ROUGE scores on the forgetting and retaining data from the \textit{TOFU} dataset using different unlearning methods on the Phi-1.5 language model. We apply the \textit{extended GDiff} with various coefficients (see \eqref{eq:static}, $0\leq c\leq 1$) and connect the results with a blue dashed line. We denote MTO methods as different markers, and use a grey dashed line to represent the loss of random guess.}
    % the $\textcolor{teal}{\bigcdot}$ marker represents method using loss normalization method, and the \textcolor{orange}{\XSolidBold{}} marker represents the result using gradient normalization method.}
    \label{fig:intro-example}
\end{figure}

Inspired by the simplicity and strong empirical performance of linear scalarization methods,\footnote{\cite{xin2022current} demonstrate that linear scalarization outperforms, or is at least on par with, other MTO approaches across various language and vision experiments} which minimize a linearly weighted average of task losses, we propose an LLM unlearning method, \algo, based on dynamic scalarization, and analyze its theoretical properties. Building on the analysis, we introduce an automatic learning rate adaptation method tailored for LLM unlearning.

%In particular, we leverage the linear scalarization as a simple and strong candidate among all multi-task methods, which minimizes a linearly weighted average of the task losses (see Eq. \eqref{eq:static}). Empirically, this formulation outperforms or at least is on par with other MTO methods in a variety of language and vision experiments \cite{xin2022current}. 

\if 0
 This is motivated by two reasons: theoretically, linear scalarization is guaranteed to be Pareto optimal (see \Cref{thm:pareto}); 
\fi 
 
% {\color{red}[we mentioned many algorithms and emphasized the linear scalarization, but it is not clear how and why we leverages this][in fact in the contribution section below, it is also not clear how scalarization is leveraged.][also from the above it appears that the main contribution is to apply those methods to unlearning, but one important contribution is that, from the lens of multi-task learning, we are developing a new 'unlearning-specific' TWO-task algorithms, the normalized GD, that are designed specifically for the unlearning problem. ZB: agreed and added in the above sentence and in the second bullet point.]}

 We showcase the effectiveness of our method through extensive experiments on multiple datasets, different LLMs and vision models. For example, on TOFU~\cite{maini2024tofu}, \algo achieves 40\% higher model utility while maintaining comparable unlearning performance with Llama2-7B. Figure \ref{fig:intro-example} highlights the effectiveness of \algo.

\vspace{2mm}
\textbf{Contributions} are summarized as follows:

\begin{compactitem}
    \item We formalize LLM unlearning as a multi-task optimization problem and unify the terminology used across both fields. We demonstrate the Pareto optimality for scalarization-based unlearning methods under some assumptions. %(\textit{cf.}, \Cref{thm:pareto!}).
    \item Through the lens of multi-task optimization, we propose a novel unlearning method \textit{\algo} for LLM unlearning, which uses the gradient norms to {\it dynamically} balance the forget and retain tasks. \algo improves both tasks \textit{simultaneously} and \textit{monotonically} with a proper learning rate scheduling. %(see \Cref{thm:NGDiff}).
    % As a consequence, the retaining utility is maintained at high level. {\color{red}[why emphasize retaining here? since we just said both will improve?]}
    \item We integrate \algo with GeN \cite{bu2024automatic}, which uses Hessian-based learning rate selection for stable convergence.
    % \item \textcolor{orange}{[ZB: Xiaomeng, can you summarize this paragraph?]}The effectiveness of our framework is validated through experiment \textbf{TODO: add results}. Compared to existing methods, our approach produces an unlearned model with high utility \textbf{TODO: add results}. We conduct extensive experiments on two unlearning datasets: Task of Fictitious Unlearning (TOFU) \cite{maini2024tofu} and Machine Unlearning Six-Way Evaluation (MUSE) \cite{shi2024muse}. We test the unlearning performance using five different LLMs, including \textsc{LLaMA2-7B-hf} \cite{touvron2023llama}, Phi-1.5 \cite{li2023textbooks}, Falcon-1b \cite{penedo2023refinedwebdatasetfalconllm}, and Mistral-7b-v0.3 \cite{jiang2023mistral}. The experimental results demonstrate that our proposed unlearning method achieves state-of-the-art unlearning performance, \textbf{TODO: add 3 results}.
\end{compactitem}

% {\color{red}[how does our result compare with the literature? since our result is pretty strong, we need to make some strong claims here in the contribution. ZB: agreed and added in the last bullet point]}

% This approach ensures that both tasks are effectively optimized.
% On the other hand, to address the learning rate sensitivity issue of gradient normalization, we incorporate an automatic unlearning strategy. Specifically, by reformulating the unlearning loss for both retain and forget data as quadratic functions, we can determine an optimal learning rate \cite{bu2024automaticgradientdescentgeneralized} which maintains high performance for retain data while maximizes the forget loss. This enables fast convergence without additional learning rate tuning.
% Furthermore, we demonstrate that the performance of \textit{GradNorm} can achieve Pareto optimality \textbf{TODO: explain} after unlearning, as illustrated in Section \ref{section:pareto}.

%Additionally, we show that the model's utility maintains the highest compared to baseline methods.
%In summary, \textit{GradNorm} achieves the following goals: (1) During the unlearning process, the performance on the retain set remains high, and the convergence rate is fast. (2) The unlearned model reaches Pareto optimality. (3) With sufficient iterations, the unlearned data are forgotten while the retained data performance maintains. (4) The unlearning process is automatic without learning rate tuning or scheduling.

\section{Related Work}
We position our work within the related literature. More discussion and background on learning rate-free techniques are in \Cref{app:related}.

\noindent\textbf{LLM unlearning}  The extensive data used in training LLMs raises significant concerns. Certain data sources contain personal information \cite{carlini2021extracting}, outdated knowledge~\cite{wu2024akewassessingknowledgeediting}, and copyright-protected materials~\cite{exhibitj}. In addition, adversarial data attacks can maliciously manipulate training data to embed harmful information~\cite{wallace-etal-2021-concealed,li2024backdoorllm}. 

To remove unwanted information without retraining the entire model, machine unlearning has been proposed using techniques such as data slicing~\cite{bourtoule2021machine}, influence functions~\cite{ullah2021machine}, and differential privacy~\cite{gupta2021adaptive}. However, these methods are challenging to scale to LLMs due to their complexity. Recently, efficient approximate unlearning methods have been proposed for LLMs~\cite{eldan2023whos,zhang2024rightforgotteneralarge,jang-etal-2023-knowledge,pawelczyk2023context,chen-yang-2023-unlearn}. They mostly focus on designing unlearning objectives or hiding unwanted information. However, none addresses the fundamental optimization problem. Our paper bridges this gap and complements existing approaches. 
Further discussion of the challenges surrounding LLM unlearning can be found in benchmarks~\cite{shi2024muse,maini2024tofu} and surveys~\cite{si2023knowledgeunlearningllmstasks,liu2024rethinkingmachineunlearninglarge}.

Note that the literature on knowledge editing~\cite{de-cao-etal-2021-editing} is also relevant. However, model editing typically focuses on surgically updating LLMs for specific knowledge, whereas unlearning removes the influence of particular documents. The techniques presented in this paper could potentially be applied to knowledge editing.

%\fix{Need another pass for this section. Flow: LLM unlearning method in NLP literature -> The need of connection to MTO -> Brief review of MTO and how they have been used in NLP -> refer more discussion to Appendix}
%This work is closely related to machine learning methods and multitask optimization. We discuss 
%and learning rate-free techniques, which are discussed in \Cref{section:pareto} and \Cref{app:related} with more details.

%\paragraph{Machine unlearning}
%\fix{need to discuss more unlearning literature for LLM. }
%Machine unlearning is oftentimes viewed as a continual learning approach, that removes specific data points after a model has been trained to memorize them. Such removal is light-weighted in contrast to re-training, especially when the forgetting set is much smaller than the retaining. In addition to the methods already introduced in \Cref{section:pareto} (namely GA, GDiff and NPO), other methods include SISA \cite{bourtoule2021machine}, influence functions \cite{ullah2021machine}, differential privacy \cite{gupta2021adaptive} and so on. However, these methods could be difficult to scale on large models and large datasets due to the algorithmic complexity. To our best knowledge, this is the first work that formulate the unlearning problem as a two-task problem, which can be solved by a number of well-known MTO methods.

\noindent\textbf{Multi-task optimization}
In NLP, multi-task learning~\cite{zhang-etal-2023-survey} typically refers to building a model that can perform well on multiple tasks simultaneously by sharing representations, introducing constraints, or combining multiple learning objectives. This often involves optimizing a form of static linear scalarization, as introduced in the next section. Multi-task optimization, on the other hand, focuses on a slightly different concept -- optimizing two distinct learning objectives simultaneously. The key challenge is how to balance the trade-off among objectives during the optimization procedure by modifying the per-task gradients (e.g. PCGrad \cite{yu2020gradient}, RLW \cite{linreasonable}, IMTL \cite{liu2021towards}). Several recent works have studied Pareto frontier and optimally in context of NLP tasks (e.g., multi-lingual machine translation~\cite{Chen2023OnTP,xin2022current} and NLP fairness~\cite{han-etal-2023-fair}). While optimization is often treated as a black-box tool in NLP research, studying optimization provides deeper insights and inspires new algorithms.\footnote{An example is the Baum-Welch algorithm, originally proposed to estimate the parameters of HMMs and applied in speech recognition in 1970s before it was recognized as an instance of the EM algorithm~\cite{DEMP1977}. It was later identified as a special case of a broader class of convex-concave optimization (CCCP)~\cite{NIPS2001_a0128693}. This connection inspired new designs, such as the unified EM~\cite{samdani-etal-2012-unified}.} %In this paper, 

%\fix{need to add other literature of MTO on other NLP tasks. }
%MTO is a paradigm where one model is trained to perform multiple tasks simultaneously, so as to significantly improve the efficiency in contrast to training multiple models, one for each task. The key challenge of MTO is the performance trade-off among tasks, where the multi-task model is worse than single-task model if trained on each task separately. Therefore, the core idea is to balance different tasks by modifying the per-task gradients, e.g. with normalization (LossNorm and NGDiff), PCGrad \cite{yu2020gradient}, RLW \cite{linreasonable}, IMTL \cite{liu2021towards}, MGDA \cite{desideri2012multiple}, CAGrad \cite{liu2024conflictaversegradientdescentmultitask}, GradVaccine \cite{wang2020gradientvaccineinvestigatingimproving}, GradDrop \cite{chen2020justpicksignoptimizing}, RotoGrad \cite{javaloy2022rotogradgradienthomogenizationmultitask}, etc.

\section{Unlearning as multi-task optimization}
This section casts machine unlearning as a multi-task  optimization (MTO), specifically the two-task optimization problem. Let the retain set be denoted by $\R$ and the forget set by $\F$, with $L_{\R}$ and $L_{\F}$ representing the corresponding cross-entropy losses for language modeling. We are interested in finding
\begin{equation}
\label{eq:mto-obj}
\arg\min\nolimits_\theta L_\R(\bm\theta) \quad \cap \quad \arg\max\nolimits_\theta L_\F(\bm\theta),
\end{equation}
where %$L$ is the cross-entropy loss and 
$\bm\theta$ represents the model parameters. 

There might not be a solution that simultaneously achieves both objectives in Eq. \eqref{eq:mto-obj}. For LLMs, the unlearning solutions generally exhibit a trade-off between performance in R and F (\textit{cf.}, \Cref{fig:intro-example}). To forget $\F$, one may unavoidably unlearn general knowledge such as grammar rules on $\F$, which can sacrifice the performance on $\R$.

In MTO, Pareto optimality is used to characterize the trade-offs between multiple objectives. In layperson's terms, if $\theta$ is Pareto optimal, it is impossible to improve $L_{\R}$ or $L_{\F}$ without worsening the other. Formal definition is in below:

%We review this classical concept in the following:
%To optimize the two tasks, let us recall the classical Pareto optimality in \Cref{def:Pareto} as there might not be a solution that can simultaneously achieve both objectives since making one loss better may hurt the other loss. 

\begin{definition}[Pareto optimality in unlearning]
\label{def:Pareto}
\looseness=-1For two models $\bm\theta$ and $\bm\theta'$, if $L_\R(\bm\theta)\geq L_\R(\bm\theta')$ and $ L_\F(\bm\theta)\leq L_\F(\bm\theta')$ with at least one inequality being strict, then $\bm\theta$ is dominated by $\bm\theta'$. A model is Pareto optimal if it is not dominated by any other models.
\end{definition}

%For LLMs, the machine unlearning solutions generally exhibit a trade-off between the performance on R and F sets (\textit{cf.}, \Cref{fig:intro-example}): without the unlearning, both R and F have high performance and high memorization; in order to forget $\F$, one may unlearn other general knowledge such as grammar rules, which can sacrifice the performance on $\R$.
In the remainder of this section, we will discuss current unlearning methods in a unified MTO framework and analyze their Pareto optimality. Building on this, we then propose a dynamic scalarization approach tailored to LLM unlearning. 

%analyze the Pareto optimality for MTO methods and We will illustrate several MTO methods and show that the Pareto optimality is guaranteed under ideal conditions.

\subsection{Static linear scalarization}

A popular MTO method is scalarization, which addresses MTO by optimizing the linear scalarization problem (LSP). This method combines multiple tasks into a single, reweighted task: 
%The scalarization method --- linearly combining multiple tasks into a single reweighted task, is arguably the most widely-used MTO method. It defines the linear scalarization problem (LSP) as
\begin{align}
\LSP(\bm\theta; c)=c\cdot L_\R(\bm\theta)-(1-c)\cdot L_\F(\bm\theta),
\label{eq:lsp}
\end{align}
where $c$ is  fixed. At iteration $t$, the gradient of LSP, $\g_\text{static}(\bm\theta_t; c)$, lies within the linear span of per-task gradients as shown in \Cref{fig:span} (yellow area):
\begin{align}
\g_\text{static}(\bm\theta_t; c) = \frac{\partial\LSP}{\partial \bm\theta_t}= c \cdot \g_\R (\bm\theta_t)-(1-c)\cdot \g_\F(\bm\theta_t).
\label{eq:static}
\end{align}
Then, the corresponding update rule by the (stochastic) gradient method is \\
\begin{align}
    \bm\theta_{t+1}=\bm\theta_t-\eta_t[c\cdot \g_\R(\bm\theta_t)-(1-c)\cdot \g_\F(\bm\theta_t)].
\end{align}

\begin{remark}
We term the static linear scalarization as the \textbf{extended GDiff} in this work. Some existing methods are special cases of extended GDiff. For example, Gradient Descent (GD) on retaining set is equivalent to extended GDiff with $c=1$. Gradient Ascent (GA) on forgetting is equivalent to that with $c=0$,
%in unlearning are GD ($c=1$), 
% when we only use $\R$ for the gradient descent and implicitly employ the catastrophic forgetting for  $\F$; (2) 
%GA ($c=0$), 
% when we only use the forggeting set for gradient ascent which is a good baseline if the performance on $\R$ does not drop too much; (3) 
and vanilla GDiff~\cite{liu2022continual} set $c=0.5$ (i.e., equally weighted). 
\end{remark}

\begin{figure}[t]
\centering
    \includegraphics[width=0.6\linewidth]{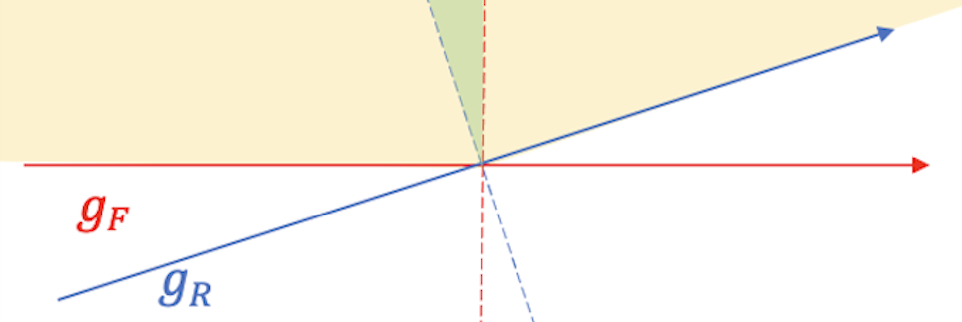}
    \caption{Gradient space in 2-dimension. $\color{red}\g_\F$ is the forgetting gradient and $\color{blue}\g_\R$ is the retaining gradient, each with a perpendicular dashed line. \textcolor{yellow}{Yellow area} is the linear span (Eq. \eqref{eq:static}) by scalarization. \textcolor{green}{Green area} is positively correlated to $\g_\R$ and negatively correlated to $\g_\F$ by Eq. \eqref{eq:sanity check}, whereas NGDiff always stays within this green area at each iteration by \Cref{fact:gnorm good}.}
    \label{fig:span}
\end{figure}

A nice property of linear scalarization is the Pareto optimality at the convergence of models, which we state in \Cref{thm:pareto} (proof in \Cref{app:proofs}) for the static $c$ and later extend to \Cref{thm:pareto!} for the dynamic $c_t$ in \Cref{section:pareto}.

\begin{lemma}[restated from \cite{xin2022current}]
For any $0<c<1$, the model $\bm\theta^*_\text{LSP}(c)\in\argmin_\theta \LSP(\bm\theta; c)$ is Pareto optimal.
\label{thm:pareto}
\end{lemma}

\Cref{thm:pareto} suggests\footnote{We note that \Cref{thm:pareto} is only applicable to the global minimum of LSP, which is not always achievable. While this result has its limitations and requires empirical validation, it provides guidance for algorithm design.}
 that we can sweep through $c\in [0,1]$ and construct the Pareto frontier after sufficiently long training time (e.g., the blue dotted line in \Cref{fig:intro-example}). However, 
 while any $c$ leads to a Pareto optimal point, the solution may be useless: e.g., perfect memorization on $(\R, \F)$ that fails to unlearn is also Pareto optimal. Next, we investigate different choices of c by extending the static scalarization in \eqref{eq:static}.

% {\color{red}[also do we need to mention another caveat in the above lemma, that 'global' optimality is required for the combined problem, which is not always possible, especially not possible in the current unlearning setting (therefore, this indicates that in practice these LSP based methods have limitations; otherwise it may be hard to justify the performance of NGDiff which is out-side of the dotted blue line, it is a dominate solution, simply because the blue dotted line is a tradeoff line but not the pareto optimal yet. )] [ZB: agreed and added in footnote.]}

% Therefore, the Pareto optimality a necessary but insufficient condition for a good unlearning solution, and $c$ must be carefully selected. Next, we investigate different choices of $c$ by extending the static scalarization in \eqref{eq:static}.

\subsection{Dynamic scalarization}
\label{section:pareto}
%In deep learning, the loss is minimized iteratively by the (stochastic) gradient method:
Static scalarization uses a constant $c$ in \eqref{eq:static}. However, we can extend it to use different scalars at different iteration:
\begin{align}
\bm\theta_{t+1}&=\bm\theta_t-\eta_t \g_\text{UN}(\bm\theta_t;c_t), \text{ where }
\nonumber \\
\g_\text{UN}(\bm\theta;c_t)&:=c_t\cdot \g_\R(\bm\theta)-(1-c_t)\cdot \g_\F(\bm\theta).
\label{eq:dynamic gradient}
\end{align}
%This results in a broad range of methods for unlearning. 

It is worth noting that instead of defining $\bm\theta^*=\text{argmin}_\theta \LSP$ at the loss level, we can define it at the gradient level based on the stationary condition of the training dynamics, i.e., $\g_\text{UN}(\bm\theta^*)=\bm{0}$. 
% {\color{red}[why this is equivalent? ZB: right, this is not equivalent, it is necessary but insufficient.]}

% \begin{align}
% c\g_\R(\bm\theta^*)-(1-c)\g_\F(\bm\theta^*)=0    
% \label{eq:stationary scalarization}
% \end{align}

% \begin{remark}
% In fact, MTO are mostly defined at the gradient level without a corresponding minimization problem. This is a common technique in optimization.
% \end{remark}

% \subsection{Dynamic multi-task optimization method}
%In light of \eqref{eq:dynamic gradient}, we summarize some unlearning and MTO methods below

Several unlearning and MTO methods can be viewed as special cases of Eq. \eqref{eq:dynamic gradient}: 
\begin{compactenum}
    \item Gradient descent (GD on $\R$), $c_t=1$
    \item Gradient ascent (GA on $\F$), $c_t=0$
    \item Gradient difference (vanilla GDiff), $c_t=0.5$
    \item Loss normalization (LossNorm), $\frac{c_t}{1-c_t}=\frac{L_\F}{L_\R}$
    \item RLW~\cite{linreasonable}, $c_t=\frac{e^{\lambda_1}}{e^{\lambda_1}+e^{\lambda_2}}$ with $\lambda_i\sim N(0,1)$
    % \item Grad normalization, $c=...$
    \item PCGrad~\cite{yu2020gradient}, $\frac{c_t}{1-c_t}=1+\frac{\g_\F^\top\g_\R}{\|g_\R\|^2}$
    \item IMTL-G~\cite{liu2021towards}, $c_t = \g_\F^\top(\frac{\g_\F}{\|\g_\F\|} - \frac{\g_\R}{\|\g_\R\|})/(\g_\F - \g_\R)^\top(\frac{\g_\F}{\|\g_\F\|} - \frac{\g_\R}{\|\g_\R\|})$
\end{compactenum}

Despite the different designs of $\{c_t\}$, we show in \Cref{thm:pareto!}  (proof in \Cref{app:proofs}) that all $\bm\theta^*(\{c_t\})$ are Pareto optimal following \Cref{thm:pareto}, including our NGDiff to be introduced in \Cref{sec:grad norm}. 

\begin{theorem}
\label{thm:pareto!}
For any $\{c_t\}$ with $0\leq c_t\leq 1$ that converges as $t\to\infty$, the model $\bm\theta^*(\{c_t\}):=\lim_{t\to\infty} \bm\theta_t$ in \eqref{eq:dynamic gradient} is Pareto optimal.
\end{theorem}

% \subsection{Bi-level method [may remove]}
% As Mingyi suggested, we have the gradient as
% $$-\nabla L(\bm\theta,R)+\lambda \pi(\nabla L(\bm\theta,R),\nabla L(\bm\theta,F))$$
% where $\pi(u,v)=v-u^\top v/|u|^2\cdot u$ is orthogonal to $u$. 
% Upon convergence, we have
% $$\lambda \nabla L(\bm\theta^*,F)-\left(1+\lambda\frac{<\nabla L(\bm\theta^*,F),\nabla L(\bm\theta^*,R)>}{|\nabla L(\bm\theta^*,R)|^2}\right)\nabla L(\bm\theta^*,R)=0$$

% With
% $$\frac{\lambda}{1+\lambda\frac{<\nabla L(\bm\theta^*,F),\nabla L(\bm\theta^*,R)>}{|\nabla L(\bm\theta^*,R)|^2}}=\frac{1-c}{c}$$
% this is equivalent to \eqref{eq:stationary scalarization} and hence Pareto optimal.

\section{Unlearning with normalized gradient difference}
While \Cref{thm:pareto!} shows the Pareto optimality of $\bm\theta^*$ as $t\to\infty$, it does not shed insight on the convergence through intermediate steps $\bm\theta_t$. Put differently, although many MTO and unlearning methods are all Pareto optimal upon convergence, they may converge to different Pareto points at different convergence speeds. Therefore, it is important to understand and control the algorithm dynamics to maintain high performance for $\R$ throughout the training. Specifically, the dynamics are determined by the choices of $\g_\text{UN}\in\mathbb{R}^d$ and $\eta_t\in\mathbb{R}$ in Eq. \eqref{eq:dynamic gradient}. 
% {\color{red}[emphasize again that they main converge to different Pareto points, so it is important to understand and control their algorithm dynamics. Below we propose a  scheme specifically designed for the unlearning problem, to control the algorithm dynamics via dynamically optimizing the coefficient c. ] [Question: do we want to say that this algorithm is designed specifically for the unlearning problem?][ZB: agreed. Does the new paragraph work?]}

In this section, we propose to use gradient normalization for $\g_\text{UN}$ and automatic learning rate for $\eta_t$, so as to achieve stable convergence, effective unlearning, high retaining utility, without manually tuning the learning rate.

\subsection{Loss landscape of unlearning}
Applying the Taylor expansion on Eq. \eqref{eq:dynamic gradient}, we can view the local landscapes of loss $L_\R$ and $L_\F$ as quadratic functions, where
\begin{equation}\label{eq: landscape}
    L_\omega(\bm\theta_{t+1})-L_\omega(\bm\theta_{t})=-\eta_t \g_\omega^\top\g_\text{UN}(c_t)
+ (\eta_t^2/2)\g_\text{UN}^\top \H_\omega \g_\text{UN}+o(\eta_t^2),
\end{equation}
where $\omega$ is either \R ~or \F. 

\if 0
\begin{align}
\begin{split}
L_\R(\bm\theta_{t+1})-L_\R(\bm\theta_{t})&= -\eta_t \g_\R^\top\g_\text{UN}(c_t)
+\frac{\eta_t^2}{2}\g_\text{UN}^\top \H_\R \g_\text{UN}+o(\eta_t^2),
\\
L_\F(\bm\theta_{t+1})-L_\F(\bm\theta_{t})&= -\eta_t \g_\F^\top\g_\text{UN}(c_t)
+\frac{\eta_t^2}{2}\g_\text{UN}^\top \H_\F \g_\text{UN}+o(\eta_t^2).
\end{split}
\label{eq: landscape}
\end{align}
\fi 

Here $\H_\omega=\frac{\partial^2 L_\omega}{\partial \bm\theta^2}$ is the Hessian matrix, which \emph{empirically} gives $\g_\text{UN}^\top \H_\omega \g_\text{UN}>0$ and renders $L_\R$ and $L_\F$ locally and directionally convex along the gradients. This allows the existence of a minimizing learning rate to be characterized in \Cref{sec:auto lr}.
% , and  {\color{red}[why we need to emphasize the point that it is >0? it is not necessary since we we don't use them here.][ZB: added reference to section 3.3, where this is needed for autoLR to exist.]}. 
We visualize the loss landscape of Phi-1.5~\cite{li2023textbooks} model on an unlearning benchmark, TOFU dataset~\cite{maini2024tofu} in \Cref{fig:parabola} and observe that the quadratic functions in Eq. \eqref{eq: landscape} are well-fitted in most iterations.

\begin{figure}[t]
   \centering
  \includegraphics[width=0.4\linewidth]{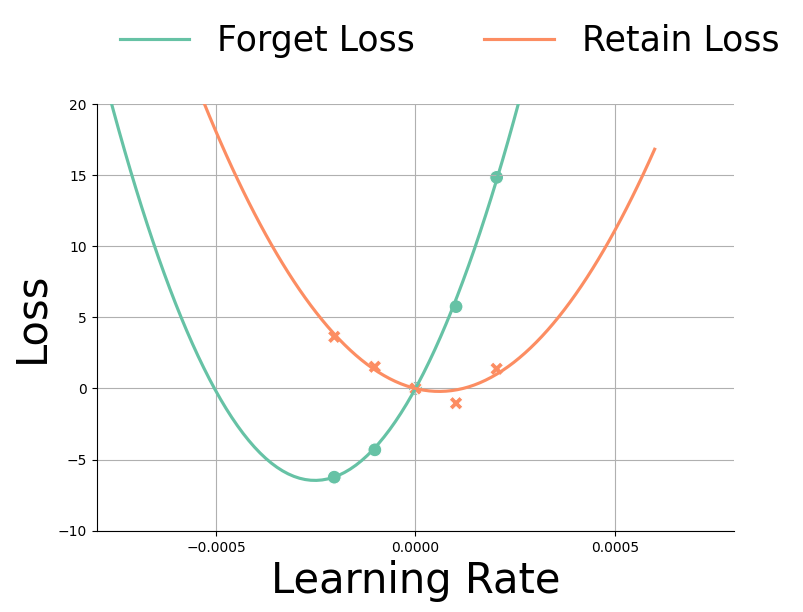}
    \caption{Loss values of retaining and forgetting sets with different learning rates. Markers are $L_\R(\bm\theta_t-\eta\g_\R)$ and $L_\F(\bm\theta_t-\eta\g_\F)$ estimated by Phi-1.5 on \textit{TOFU} at step 10. The curves are fitted as quadratic functions.}
    \label{fig:parabola}
\end{figure}
\if0
\begin{figure*}[!htb]
\RawFloats
\begin{minipage}{.45\linewidth}

\end{minipage}
\hspace{.05\linewidth}
\begin{minipage}{.44\linewidth}
\centering
  \includegraphics[width=\linewidth]{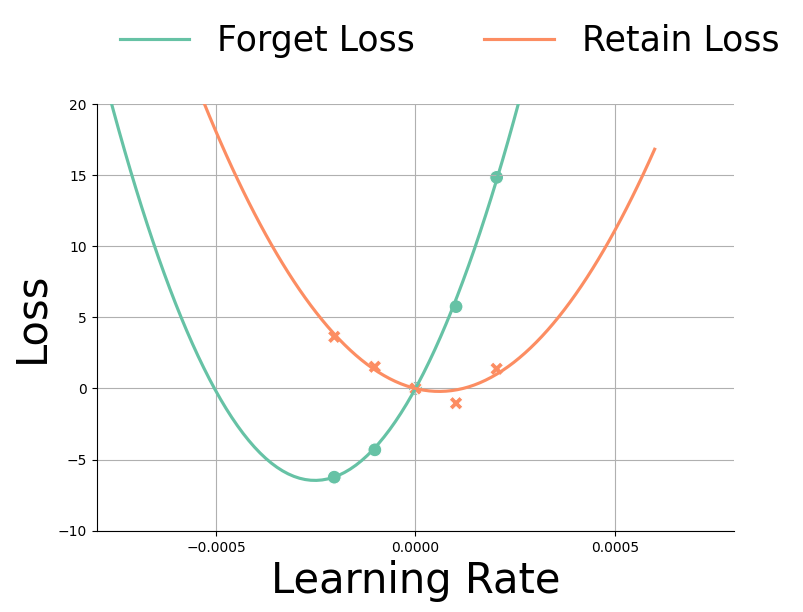}
    \caption{Loss values of retaining and forgetting sets with respect to different learning rates. Markers are $L_\R(\bm\theta_t-\eta\g_\R)$ and $L_\F(\bm\theta_t-\eta\g_\F)$ using the \textit{TOFU} dataset on Phi-1.5 at step 10, on which the curves are fitted as quadratic functions.}
    \label{fig:parabola}
\end{minipage}
\end{figure*}

\fi

\subsection{Normalized gradient difference}
\label{sec:grad norm}
In order for $L_\F$ to increase as well as $L_\R$ to decrease, we want to construct $\g_\text{UN}$ such that
\begin{align}
\g_\R^\top\g_\text{UN}(c_t) \geq 0\geq\g_\F^\top\g_\text{UN}(c_t).
\label{eq:sanity check}
\end{align}

To satisfy Eq. \eqref{eq:sanity check} , we propose a normalized gradient difference method (\algo) to dynamically set
$c_t=\frac{1/\|\g_\R\|}{1/\|\g_\R\|+1/\|\g_\F\|} \Longrightarrow \g_\textup{NGDiff}(\g_\R,\g_\F):=\frac{\g_\R}{\|\g_\R\|}-\frac{\g_\F}{\|\g_\F\|}.$ In words, we normalize the retaining and forgetting gradients\footnote{We illustrate in \Cref{app:gradnorm} that NGDiff is critically different and simpler than GradNorm.\citep{chen2018gradnorm}.}.
% \begin{align*}
% \text{NGDiff: }
% \theta_{t+1}=\theta_t-\eta_t\left[\frac{\g_\R}{|\g_\R|}-\frac{\g_\F}{|\g_\F|}\right]
% \end{align*}
%and state that \eqref{eq:sanity check} is satisfied at all iterations in \Cref{fact:gnorm good}.
% \footnote{In fact, other methods such as GDiff, GD, GA, LN, PCgrad, GradNorm can find counter-examples.\textcolor{red}{[ZB: if we want to keep this, we need to make examples in appendix.]}}
We analyze \algo as follows. 
First, we show the condition in Eq. \eqref{eq:sanity check} is satisfied at all iterations in the following lemma (proof in \Cref{app:proofs}):
\begin{lemma}
$\g_\textup{NGDiff}(\g_\R,\g_\F)$ satisfies Eq. \eqref{eq:sanity check} for any $\g_\R\in\mathbb{R}^d$ and $\g_\F\in\mathbb{R}^d$. In words, $\g_\textup{NGDiff}$ is always positively correlated to $\g_\R$ and negatively correlated to $\g_\F$.
\label{fact:gnorm good}
\end{lemma}

% \begin{remark}
% Note other $\{c_t\}$ may also work if and only if $\forall t, \g_\R,\g_\F$,
% $$c_t|g_\R|^2-(1-c_t)\langle g_\R,g_\F\rangle>0$$
% $$c_t\langle g_\R,g_\F\rangle-(1-c_t)|g_\F|^2>0$$
% Depending on the sign and relative magnitude of $\langle g_\R,g_\F\rangle$, the constraints can take various forms.
% \end{remark}

% In \Cref{thm:NGDiff}, we can leverage \Cref{fact:gnorm good} to claim the effectiveness of forgetting (i.e. $L_\F$ increases) and retaining (i.e. $L_\R$ decreases) under the appropriate learning rate, which is to be implemented automatically in \Cref{sec:auto lr}.
In \Cref{thm:NGDiff} (proof in \Cref{app:proofs}), we leverage \Cref{fact:gnorm good} to claim that the local loss improvement under appropriate learning rate, which will be implemented adaptively in \Cref{sec:auto lr}.

\begin{theorem}
Consider $\bm\theta_{t+1}=\bm\theta_t-\eta \g_\textup{NGDiff}$. 

(1) Unless $\g_\R$ is exactly parallel to $\g_\F$, for any sufficiently small learning rate $\eta$, there exist two constants $\epsilon_{\R,1}=o(\eta), \epsilon_{\F,1}=o(\eta)$ such that
\begin{equation*}
\begin{split}
    L_\R(\bm\theta_{t+1})-L_\R(\bm\theta_{t})&< \epsilon_{\R,1}; \\
    L_\F(\bm\theta_{t+1})-L_\F(\bm\theta_{t})&> \epsilon_{\F,1}.
\end{split}
\end{equation*}

(2) If additionally $\g_\textup{NGDiff}^\top\H_\R\g_\textup{NGDiff}>0$ and $\g_\textup{NGDiff}^\top\H_\F\g_\textup{NGDiff}>0$, then for any learning rate $0<\eta<\frac{2\g_\R^\top\g_\textup{NGDiff}}{\g_\textup{NGDiff}^\top \H_\R \g_\textup{NGDiff}}$, 
there exist two constants $\epsilon_{\R,2}=o(\eta^2), \epsilon_{\F,2}=o(\eta^2)$ such that 
\begin{equation*}
\begin{split}
L_\R(\bm\theta_{t+1})-L_\R(\bm\theta_{t})&< \epsilon_{\R,2};\\ L_\F(\bm\theta_{t+1})-L_\F(\bm\theta_{t})&> \epsilon_{\F,2}.
\end{split}
\end{equation*}
\label{thm:NGDiff}
\end{theorem}

To interpret \Cref{thm:NGDiff}, we view $\epsilon\approx 0$ as $\eta$ is generally small (say $\eta\sim 10^{-4}$ in our experiments), and hence, NGDiff is optimizing on $\R$ and $\F$ simultaneously. Visually speaking, \Cref{fact:gnorm good} constrains NGDiff's gradient to stay in the green area in \Cref{fig:span} unless $\g_\F\parallel\g_\R$, whereas other methods do not explicitly enforce Eq. \eqref{eq:sanity check} and may consequently harm the retaining utility.

We end the analysis with the following remark: 

\begin{remark}
The condition, $\g_\textup{NGDiff}^\top\H\g_\textup{NGDiff}>0$ in part (2) of \Cref{thm:NGDiff}, may not always hold in deep learning. However, it empirically holds in most iterations across models and datasets in our experiments (\textit{cf.}, our \Cref{fig:parabola} and Figure 2 in \cite{bu2024automatic}), and we can stablize the training by not updating $\eta$ when the condition fails.
\end{remark}
% {\color{red}[I think the fact that we require $g H g >0$ will invite some attack. It maybe better to: 1) put some remark after T5 to justify this fact. Some numerical evidence, or citation saying that along the iteration this is always true is good to have; 2) in the Theorem, we can state two versions: a) your correct version, which will give you a non-zero descent, b) unless $g_r$ and $g_f$ points towards exactly the same direction, in which case, there is nothing to learn. Will this be a better way to state the theorem?][ZB: second point added to Theorem 5, first point added in Remark 3.1]}

\subsection{Automatic learning rate adaption}
\label{sec:auto lr}

In order for NGDiff to work as in \Cref{thm:NGDiff}, the learning rate schedule needs to be carefully selected so that $0<\eta_t<\frac{2\g_\R^\top\g_\textup{NGDiff}}{\g_\textup{NGDiff}^\top \H_\R \g_\textup{NGDiff}}$ at each iteration. In \Cref{alg:auto unlearning}, we adapt GeN~\citep{bu2024automatic} (or AutoLR)  to the unlearning setting and dynamically set the learning rates\footnote{We note other parameter-free methods such as D-adaptation, Prodigy, and DoG can also set the learning rate automatically. However, these methods need to be tailored for different gradient methods, hence not compatible to NGDiff or the unlearning algorithms in general. We give a detailed explanation in \Cref{app:related}.} as the minimizer of \eqref{eq: landscape}: to locally optimize $L_\R$ and to monotonically increase $L_\F$, we use the following learning rate:
\begin{align}
\eta_t^*={\g_\R^\top \g_\textup{NGDiff}}/{\g_\textup{NGDiff}^\top \H_\R \g_\textup{NGDiff}}.
\label{eq:optim lr}
\end{align} 
% We may alternatively choose $\eta_t=2\eta_t^*$, so as to maintain $L_\R$ but accelerate the forgetting by increasing $L_\F$.

 GeN estimates two scalars -- the numerator and denominator of Eq. \eqref{eq:optim lr} by analyzing the difference of loss values, thus the high-dimensional Hessian matrix $\H_\R$ is never instantiated.
 We devote \Cref{app:GeN} to explain how GeN works and how we have modified GeN for unlearning, such as only forward passing on $\R$ but not $\F$ in Eq. \eqref{eq:optim lr}. 

% {\color{red}[elaborate more on how the rhs is estimated.][ZB: done in appendix \Cref{app:GeN}]}

\begin{remark}
\label{rem:GeN lazy}
There is a computational overhead to use GeN, as it requires additional forward passes to estimate $\eta_t^*$. Nevertheless, we only update the learning rate every 10 iterations so that the overhead is amortized and thus negligible.
\end{remark}

\Cref{alg:auto unlearning} summarizes \algo. We note that $\g_\F,\g_\R$ can be stochastic gradients.

\begin{algorithm}[!htb]

\caption{Learning-rate-free NGDiff}
\begin{algorithmic}[1]
\For{$t=1,2,...$}
\State \textcolor{gray}{----NGDiff----}
\State Compute  $L_\R(\bm\theta_t)$ by a forward pass on $\R$
\State Compute $\g_\R(\bm\theta_t)$ by backward propagation
\State Compute  $L_\F(\bm\theta_t)$ by a forward pass on $\F$
\State Compute $\g_\F(\bm\theta_t)$ by backward propagation
\State Construct $\g_\textup{NGDiff}=\g_\R/\|\g_\R\|-\g_\F/\|\g_\F\|$
\State \textcolor{gray}{----AutoLR----}
\If{$t$ mod $10==0$:}
\State Compute $ L_{R}^\pm= L_\R(\bm\theta_t \pm\eta\g)$ by two forward passes on $\R$
\State Fit the quadratic function $\mathcal{Q}$ in Eq. \eqref{eq: landscape} from $(-\eta,0,\eta)\to ( L_\R^-,  L_\R,  L_\R^+)$
\State Derive the optimal learning rate $\eta_t^*:=\text{argmin}_x\mathcal{Q}(x)$ and set $\eta=\eta_t^*$
\EndIf
\State Update $\theta_{t+1}=\theta_t-\eta \g_\textup{NGDiff}$
\EndFor
\end{algorithmic}
\label{alg:auto unlearning}
\end{algorithm}

\section{Experiments}
\label{sec:exp}
\subsection{Setup}
\label{sec:dataset}
\noindent\textbf{Dataset} We evaluate the empirical performance of our proposed method on the two following datasets (see more dataset details in Section \ref{sec:app_dataset}):

\begin{table*}[t]
    \centering
    \small
    \vspace{-0.2cm}    
    \caption{%\textit{Verbatim memorization}, \textit{Model utility}, and \textit{TruthRatio} 
    Performance on \textit{TOFU} dataset (\textit{forget10/retain90)} with different unlearning methods and models. 
  We define success as the model being able to reduce \textit{Verbatim memorization} to below 0.1 or maintain at least 70\% of the \textit{Model utility} and the \textit{TruthRatio}  compared to the \textit{No-unlearn}, with successful cases highlighted in bold. \algo achieves success in most cases. }  
    
  %  Results with lower \textit{Verbmem} with higher \textit{Utility} and \textit{TruthRatio} are superior.}
    \setlength{\tabcolsep}{4pt}
    \begin{tabular}{c|c|c|cccccc}
        \toprule
        \multirow{2}{*}{\textbf{Base Model}} 
 &\multirow{2}{*}{\textbf{Metric}} & \multicolumn{7}{c}{\textbf{Method}}\\
	& & \textit{No-unlearn} & \textit{GDiff-0.9} & \textit{GDiff-0.5} & \textit{GDiff-0.1} & \textit{NPO} & \textit{LossNorm} & \textit{NGDiff} \\
        \midrule
        \multirow{3}{*}{Phi-1.5} & Verbmem $\downarrow$ & $1.000$ & $0.805$ & $\textbf{0.027}$ & $\textbf{0.000}$ & $\textbf{0.000}$ & $0.432$ & $\textbf{0.024}$ \\
        & Utility $\uparrow$ & $1.000$ & $0.992$ & $0.308$ & $0.000$ & $0.000$ & $\textbf{0.752}$ & $\textbf{0.747}$ \\
        & TruthRatio $\uparrow$ & $0.385$ & $0.205$ & $0.216$ & $0.221$ & $0.179$ & $0.214$ & $\textbf{0.353}$ \\
         \midrule
         \multirow{3}{*}{Falcon-1B} & Verbmem $\downarrow$ & $1.000$ & $\textbf{0.041}$ & $\textbf{0.001}$ & $\textbf{0.000}$ & $\textbf{0.017}$ & $\textbf{0.055}$ & $\textbf{0.021}$ \\
        & Utility $\uparrow$ & $1.000$ & $\textbf{0.434}$ & $0.305$ & $0.000$ & $0.114$ & $0.521$ & ${0.428}$ \\
        & TruthRatio $\uparrow$ & $0.408$ & $0.237$ & $0.244$ & $0.217$ & $0.184$ & $0.252$ & $\textbf{0.354}$ \\
         \midrule
         \multirow{3}{*}{GPT2-XL} & Verbmem $\downarrow$ & $1.000$ & $\textbf{0.029}$ & $\textbf{0.001}$ & $\textbf{0.000}$ & $\textbf{0.031}$ & $\textbf{0.022}$ & $\textbf{0.046}$ \\
        & Utility $\uparrow$ & $0.999$ & $0.381$ & $0.250$ & $0.000$ & $0.136$ & $0.376$ & $\textbf{0.792}$ \\
        & TruthRatio $\uparrow$ & $0.412$ & $0.186$ & $0.278$ & $0.133$ & $0.179$ & $0.196$ & $\textbf{0.399}$ \\
         \midrule
         \multirow{3}{*}{Llama2-7B} & Verbmem $\downarrow$ & $1.000$ & $0.810$ & $\textbf{0.011}$ & $\textbf{0.000}$ & $0.709$ & $\textbf{0.010}$ & $\textbf{0.002}$ \\
        & Utility $\uparrow$ & $1.000$ & $\textbf{0.851}$ & $0.324$ & $0.000$ & $0.682$ & $0.264$ & $\textbf{0.724}$ \\
        & TruthRatio $\uparrow$ & $0.490$ & ${0.340}$ & $\textbf{0.364}$ & $0.161$ & $0.329$ & $0.329$ & ${0.334}$ \\
         \midrule
         \multirow{3}{*}{Mistral-7B} & Verbmem $\downarrow$ & $1.000$ & $1.000$ & $0.945$ & $0.410$ & $0.385$ & $0.259$ & $\textbf{0.009}$ \\
        & Utility $\uparrow$ & $1.000$ & $\textbf{0.999}$ & $\textbf{0.944}$ & $0.517$ & $0.341$ & $\textbf{0.925}$ & $\textbf{0.996}$ \\
        & TruthRatio $\uparrow$ & $0.344$ & $\textbf{0.345}$ & $\textbf{0.366}$ & $\textbf{0.374}$ & $\textbf{0.364}$ & $\textbf{0.358}$ & $\textbf{0.379}$ \\
        \bottomrule
    \end{tabular}
\label{table:results_unlearning_tofu}
\end{table*}

\textit{Task of Fictitious Unlearning (TOFU)} \citep{maini2024tofu}. \textit{TOFU} consists of 20 question-answer pairs based on fictitious author biographies generated by GPT-4 \citep{achiam2023gpt}.
    % The dataset is divided into a retaining set and a forgetting set. 
    In our experiments, we use the \textit{forget10} (10\% of the full training set) as the forgetting set and \textit{retain90} (90\% of the full training set) as the retaining set.

 \textit{MUSE-NEWS} \citep{shi2024muse}. This dataset consists of BBC news articles \citep{li2023avoiding} published since August 2023. We use its \textit{train} split to finetune a target model, and then the \textit{raw} set, which includes both the forgetting and retaining data, for the target model unlearning. Finally, the \textit{verbmem} and \textit{knowmem} splits are used to evaluate the unlearned model's performance.

% \textcolor{red}{[ZB: is this needed?]In our experiments, we utilize the above datasets for finetuning, unlearning, and evaluating the target model. The detailed statistics of the datasets are presented .}

\noindent\textbf{Unlearning methods} We compare \algo with 4 baselines. The first baseline method is the target model without any unlearning, while the remaining three are the state-of-the-art unlearning methods.

\textit{No-unlearn}. We fine-tune the base model on the full training data. Subsequent unlearning approaches are then applied on  \textit{No-unlearn}.
% \textit{Retain-data-only}. This method utilizes only the existing retain data to finetune the pretrained model. The retrained model does not have any knowledge from the forget set and should have the ideal performance.

\textit{Gradient Difference (GDiff)} \citep{liu2022continual}. 
GDiff  (see Sec. \ref{section:pareto}) applies
static linear scalarization with $c=0.5$ in MTO. For a thorough comparison, we also include the extended GDiff method, with $c=0.1$ or $c=0.9$.
% Augmented from the Gradient Ascent method, which performs gradient ascent on the cross-entropy loss of the forget set, this approach also preserves model performance on the retain set by incorporating a standard gradient descent on the cross-entropy loss of the retain data. 
% The loss function is defined as follows:

% \begin{equation}
% L_{diff} = - \eta \cdot L(S_f, w) + (1 - \eta) \cdot L(S_r, w)
% \end{equation}

% where $S_r$ are randomly sampled from $D_r$, $S_f$ are randomly sampled from $D_f$, and $\eta$ is a constant and $\eta \in [0, 1]$.

% \textit{Scalable Remembering and Unlearning Unbound (SCRUB)} \cite{kurmanji2024towards}. In this approach, the goal is to efficiently forget  without hurting model utility. This unlearning method takes the original model as the \textit{teacher} model and trains a \textit{student} model that selectively obeys that teacher. The training objective is to minimize the KL divergence between the teacher and student output distributions for retain examples, and maximize that for forget examples.
% \\

\textit{Loss Normalization (LossNorm)}. As discussed in Section \ref{section:pareto}, this approach computes and normalizes the forget loss and retain loss separately, with the overall loss being $L_\R/|L_\R|-L_\F/|L_\F|$.

% \begin{equation}
% L_{diff} = - \eta \cdot \frac{L(S_f, w)}{\norm{L(S_f, w)}} + (1 - \eta) \cdot \frac{L(S_r, w)}{\norm{L(S_r, w)}}
% \end{equation}

% where $S_r$ are randomly sampled from $D_r$, $S_f$ are randomly sampled from $D_f$, and $\eta$ is a constant with $\eta \in [0, 1]$.

\textit{Negative Preference Optimization (NPO)} \citep{zhang2024negative}. %GA method often struggles to effectively unlearn the forgetting data, thus resulting in significant degradation in the model's performance on the retaining data. The NPO approach addresses this issue based on 
NPO uses preference optimization \citep{ouyang2022traininglanguagemodelsfollow} with the loss: $L_{\text{NPO}, \beta}(\bm\theta) =$
\begin{equation}
 -\frac{2}{\beta}\E_{\F} \Big[\log \sigma \Big(-\beta \log \frac{f(S, w)}{f_\text{No-unlearn}(S, w)}\Big) \Big],
\end{equation}
where $S$ is randomly sampled from F, $\beta > 0$ is the inverse temperature, $f$ is the unlearned model, and $f_\text{No-unlearn}$ is the model before unlearning.

\paragraph{Foundation Models}
We test multiple LLMs: \textsc{LLaMA2-7B} \citep{touvron2023llama}, Phi-1.5 \citep{li2023textbooks}, Falcon-1B \citep{penedo2023refinedwebdatasetfalconllm}, GPT2-XL \citep{radford2019language} and Mistral-7B \citep{jiang2023mistral}. They are pre-trained and then fine-tuned on datasets in Section \ref{sec:dataset}, with AdamW optimizer and are carefully tuned (Appendix \ref{sec:app_experiment}).

\begin{table*}[!htb]
\vspace{-0.2cm}       \centering
    \small
    \setlength{\tabcolsep}{4pt}
    \begin{tabular}{c|c|c|cccccc}
        \toprule
        \multirow{2}{*}{\textbf{Base Model}} 
 &\multirow{2}{*}{\textbf{Metric}} & \multicolumn{7}{c}{\textbf{Method}}\\
	& & \textit{No-unlearn} & \textit{GDiff-0.9} & \textit{GDiff-0.5} & \textit{GDiff-0.1} & \textit{NPO} & \textit{LossNorm} & \textit{NGDiff} \\
        \midrule
         \multirow{3}{*}{Llama2-7B} & Verbmem $\downarrow$ & $0.561$ & $0.555$ & $\textbf{0.043}$ & $\textbf{0.004}$ & $\textbf{0.000}$ & $0.388$ & $\textbf{0.036}$ \\
           & Knowmem $\downarrow$ & $0.755$ & $0.717$ & $\textbf{0.287}$ & $\textbf{0.000}$ & $\textbf{0.000}$ & $\textbf{0.514}$ & $\textbf{0.455}$ \\
        & Utility $\uparrow$ & $0.646$ & $\textbf{0.641}$ & $0.275$ & $0.000$ & $0.000$ & $\textbf{0.506}$ & $\textbf{0.556}$ \\
         \midrule
         \multirow{3}{*}{Mistral-7B} & Verbmem $\downarrow$ & $0.578$ & $0.177$ & $\textbf{0.000}$ & $\textbf{0.000}$ & $0.113$ & $0.196$ & $\textbf{0.098}$ \\
           & Knowmem $\downarrow$ & $0.416$ & $\textbf{0.257}$ & $\textbf{0.000}$ & $\textbf{0.000}$ & $0.343$ & $0.293$ & $\textbf{0.165}$ \\
        & Utility $\uparrow$ & $0.411$ & $\textbf{0.339}$ & $0.000$ & $0.000$ & $\textbf{0.316}$ & $\textbf{0.343}$ & $\textbf{0.354}$ \\
      
        \bottomrule
    \end{tabular}
\vspace{-0.2cm}    \caption{
Results on the \textit{MUSE-NEWS} dataset. 
We boldface the entries where unlearning successfully reduces \textit{Verbatim memorization} to below 0.1, reduces \textit{Knowledge memorization} to less than 70\% of \textit{No-Unlearn}, or maintains at least 70\% of the Utility compared to \textit{No-Unlearn}. With the exception of \algo, most unlearning approaches exhibit a significant trade-off between forgetting and utility.}

%We boldface \textit{Verbatim memorization}, \textit{Model utility}, and \textit{Knowledge memorization} on the \textit{MUSE-NEWS} dataset with different unlearning methods on Llama2-7B and Mistral-7B models. Results with lower \textit{Verbmem} and \textit{Knowmem}  with higher \textit{Utility} are superior.}
\label{table:results_unlearning_muse}
\end{table*}

\subsection{Evaluation Metrics}

Following the existing work \citep{shi2024muse}, we evaluate the unlearning performance based on model's output quality. We expect a good performance should satisfy the following requirements:

\noindent\textbf{No verbatim memorization}. After the unlearning, the model should no longer remember any verbatim copies of the texts in the forgetting data. To evaluate this, we prompt the model with the first $k$ tokens in F and compare the model's continuation outputs with the ground truth continuations. We use ROUGE-L recall scores for this comparison, where a lower score is better for unlearning.

\noindent\textbf{No knowledge memorization}. After the unlearning, the model should not only forget verbatim texts, but also the knowledge in the forgetting set. For the \textit{MUSE-NEWS} dataset, we evaluate knowledge memorization using the \textit{Knowmem$_\F$} split, which consists of generated question-answer pairs based on the forgetting data. Similar to verbatim memorization, we use ROUGE-L recall scores. % for the evaluation.

\noindent\textbf{Maintained model utility}. An effective unlearning method must maintain the model's performance on the retaining set. We prompt the model with the question from R and compare the generated answer to the ground truth. We use ROUGE-L recall scores for these comparisons. Additionally, we evaluate the model using the Truth Ratio metric. We use the \textit{Retain10-perturbed} split from \textit{TOFU}, which consists of five perturbed answers created by modifying the facts in each original answer from R. The Truth Ratio metric computes how likely the model generates a correct answer versus an incorrect one, where a higher value is better.

\begin{table}[t]
    \centering
    \small
    \setlength{\tabcolsep}{6pt}
    \begin{tabular}{@{}c|ccc@{}}
        \toprule
        \multirow{2}{*}{\textbf{Method}} & \multicolumn{3}{c}{\textbf{TOFU} (without $\to$ with AutoLR)} \\
	& \textit{Verbmem} $\downarrow$  & \textit{Utility} $\uparrow$ & \textit{TruthRatio} $\uparrow$ \\
        \midrule
        \textbf{No-unlearn} & $1.00$ & $1.00$ & $0.39$ \\
        \textbf{GDiff c=0.9} & $0.81 \to 0.20$ & $0.99 \to 0.42$ & $0.21 \to 0.31$ \\
        \textbf{GDiff c=0.5} & $0.03 \to 0.00$ & $0.31 \to 0.03$ & $0.22 \to 0.30$ \\
        \textbf{GDiff c=0.1} & $0.00 \to 0.00$ & $0.00 \to 0.00$ & $0.22 \to 0.23$ \\
        \textbf{NPO} &  $0.00 \to 0.00$ & $0.00 \to 0.00$ & $0.18\to 0.22$ \\
        \textbf{LossNorm} &  $0.43 \to 0.23$ & $0.75 \to 0.73$ & $0.21 \to 0.34$ \\
        \textbf{NGDiff} & $0.02 \to 0.01$ & $0.61 \to 0.75$  & $0.29 \to 0.35$ \\
        \bottomrule
    \end{tabular}
    \caption{%Results of \textit{Verbmem}, \textit{Utility}, and \textit{TruthRatio} using 
    Influence of AutoLR with 
    different unlearning methods on the Phi-1.5 model. AutoLR improves the \textit{TruthRatio} and reduces \textit{Verbmem} across all methods. W/ or w/o AutoLR, \algo  outperforms other baselines.}
    \label{table:ablation_ngdiff}
\end{table}

\subsection{Main Results}

The results for \textit{Verbatim memorization (Verbmem)}, \textit{Model utility (Utility)}, \textit{TruthRatio}, and \textit{Knowledge memorization (Knowmem)} using different unlearning methods are presented in Table \ref{table:results_unlearning_tofu}, \ref{table:results_unlearning_muse} as well as  \ref{table:results_unlearning_muse_appendix} in Appendix. We evaluate these metrics using \textit{TOFU} and \textit{MUSE-NEWS} across LLMs.

In summary, our NGDiff consistently achieves the superior performance across all models on both datasets. In stark contrast, the baseline unlearning methods (1) either effectively forget $\R$ by reducing \textit{Verbmem} and \textit{Knowmem} but fail maintain the \textit{Utility} and \textit{TruthRatio}, such as GDiff with $c\leq 0.5$, NPO; (2) or cannot unlearn $\F$ on Phi-1.5 and Mistral-7B, such as LossNorm and GDiff with $c=0.9$. We highlight that the effectiveness of these unlearning methods are highly model-dependent and dataset-dependent, unlike NGDiff.

\begin{figure*}[!htb]
\centering
% \includegraphics[width=0.36\linewidth]{figures/all_method_rouge_retain_step.png}
% \includegraphics[width=0.36\linewidth]{figures/all_method_rouge_forget_step.png}
% % \includegraphics[width=0.27\linewidth]{figures/all_method_rouge_forget_retain_last_step.png}
% \includegraphics[width=0.36\linewidth]{figures/all_method_rouge_forget_trajectory.png}
% \includegraphics[width=0.36\linewidth]{figures/all_method_loss_forget_trajectory.png}
\includegraphics[width=0.99\linewidth]{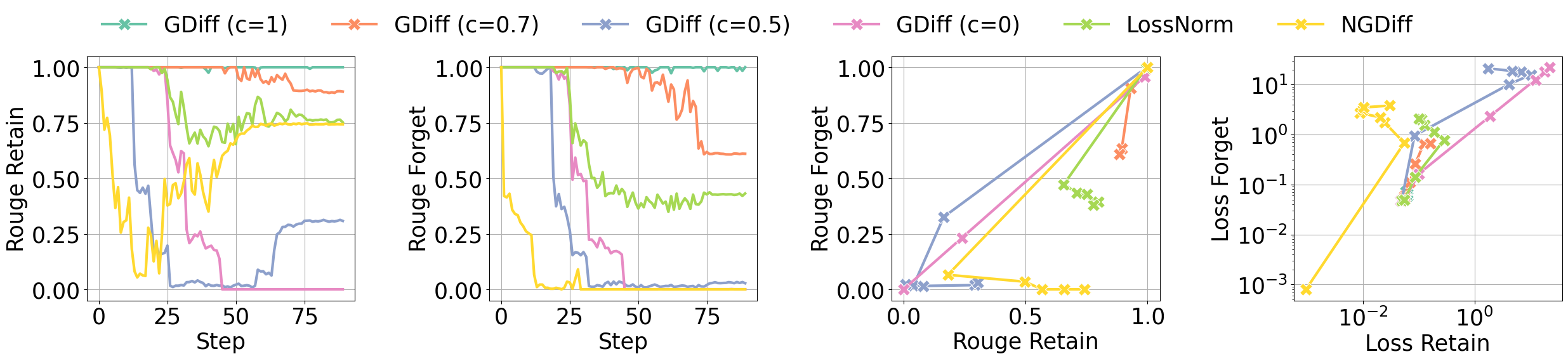}
\vspace{-0.3cm}
\caption{Comparison of unlearning methods on TOFU. The figures show the ROUGE scores and loss terms during unlearning process with different methods, which includes GDiff, LossNorm, and NGDiff. We observe that NGDiff effectively unlearns the forgetting data while maintaining the performance on the retaining data.}
% top-left: Dynamic weights improves over static GD [all methods, retain rouge v.s. iterations]. top-right: [all methods, forget rouge v.s. iterations]. bottom-left: [all methods, forget v.s. retain rouge, plot multiple epochs]. bottom-right: [all methods, forget v.s. retain rouge, plot multiple epochs].}
\label{fig:all_method_rouge_loss_steps}
\end{figure*}

\begin{figure*}[!htb]
\vspace{-0.4cm}
\centering
\includegraphics[width=.99\linewidth]{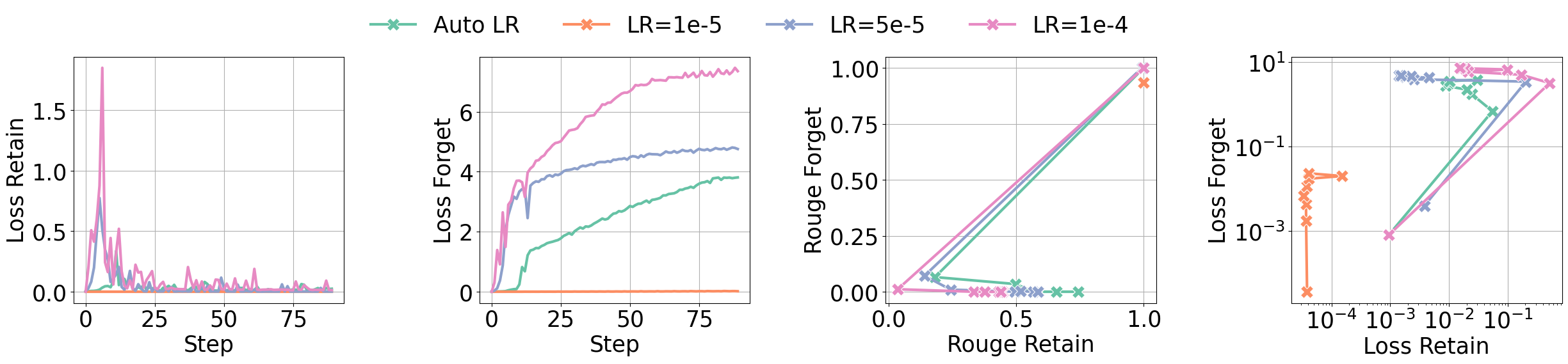}
% \caption{GradNorm is sensitive to lr [multiple lr AND auto lr, left: retain rouge v.s. iterations; right: forget rouge v.s. iterations; maybe another plot in Pareto and show two versions of GradNorm but this can wait.]}
\vspace{-0.3cm}
\caption{Comparison between AutoLR and different learning rates on NGDiff. The figures show the ROUGE scores and loss values during the unlearning process on TOFU dataset using Phi-1.5 model. We observe that AutoLR outperforms the static learning rates with better model utility and more stable convergence.}
\label{fig:ablation_autolr}
\end{figure*}

For the \textit{TOFU} dataset, we observe that some unlearning methods fail to unlearn the forget data effectively. For example, GDiff-0.9 and LossNorm do not unlearn effectively when applied to Phi-1.5, Llama2-7B and Mistral-7B. In fact, GDiff-0.9 has $80\%\sim100\%$ \textit{Verbmem} and LossNorm has $>40\%$ \textit{Verbmem} on Phi-1.5. However, they are effective on Falcon-1B and GPT2-XL, even though these models have similar sizes ($\approx 1$B parameters) to Phi-1.5. On the other hand, some methods fail to preserve the model utility after unlearning. For example, GDiff-0.1 has close to 0 \textit{Utility} on Phi-1.5, Falcon-1B, GPT2-XL and Llama2-7B; similarly, NPO also experiences a significant drop in \textit{Utility} on Phi-1.5 model, Falcon-1B and GPT2-XL, but not so on Llama2-7B. In contrast, our NGDiff remains effective in unlearning $\F$ and maintaining $\R$ across the models. In addition, NGDiff achieves the best \textit{TruthRatio} on all models except Llama2-7B (which is still on par with the best), indicating that the model's answers remain factually accurate for questions in the retaining data.

For the \textit{MUSE-NEWS} dataset, NGDiff also outperforms the baseline methods on Llama2-7B and Mistral-7B models by achieving a lower \textit{Verbmem} and a higher \textit{Utility}. The \textit{Knowmem} results indicate that NGDiff not only unlearns the verbatim copies of the forgetting texts, but also successfully removes the associated knowledge. While the model capacities of Phi-1.5 and Falcon-1B are smaller, limiting their ability to learn knowledge effectively after fine-tuning on the full dataset, as shown in \Cref{table:results_unlearning_muse_appendix},
% As shown in Table \ref{table:results_unlearning_muse_appendix}, both models have much lower performance across all three metrics even before unlearning. 
\algo still performs well. 

%successfully unlearns the forgetting data while preserving model utility.

To further illustrate the performance of our proposed method during the training, in addition to the last iterate results, we plot the ROUGE scores and loss terms during the unlearning process in  \Cref{fig:all_method_rouge_loss_steps}. We apply the extended GDiff, LossNorm, and NGDiff methods, to the Phi-1.5 model using the \textit{TOFU} dataset. While GDiff with $c=0.5$ and $c=0.7$, and NGDiff are effective in unlearning, only NGDiff preserve the model utility above $75\%$ ROUGE score. A closer look at the second and the fourth plots of \Cref{fig:all_method_rouge_loss_steps} shows that NGDiff exhibits the fastest and most stable convergence on $\F$ while maintaining a low retaining loss $\leq 0.1$. 
% \textcolor{green}{[ZB: do we want the following sentence? "We note our optimization objectives in \eqref{eq:lsp} are the forgetting and retaining losses, hence the behaviors of ROUGE scores are not guaranteed."}

\subsection{Ablation Study}
\textbf{Effectiveness of NGDiff}.
In our experiments, we utilize the automatic learning rate scheduler (AutoLR) for NGDiff method. To investigate the impact of NGDiff alone, we compare all methods with or without AutoLR in \Cref{table:ablation_ngdiff}. With AutoLR or not (where we use manually tuned learning rates), NGDiff, GDiff ($c=0.1$ or $0.5$) and NPO can effectively unlearn in terms of \textit{Verbmem}. However, among these four methods, NGDiff uniquely retains a reasonable \textit{Utility} between $60\sim 75\%$, while other methods retains only $0\sim 30\%$ \textit{Utility}. A similar pattern is observed in terms of \textit{TruthRatio} as well. Overall, NGDiff significantly outperforms other baseline methods with or without AutoLR.

\if0

\begin{table}[t]
    \centering
    \small
    \setlength{\tabcolsep}{6pt}
    \begin{tabular}{c|ccc}
        \toprule
        \multirow{2}{*}{\textbf{Method}} & \multicolumn{3}{c}{\textbf{TOFU} (wo $\to$ w/ AutoLR)} \\
	& \textit{Verbmem} $\downarrow$  & \textit{Utility} $\uparrow$ & \textit{TruthRatio} $\uparrow$ \\
        \midrule
        \textbf{No-unlearn} & $1.000$ & $1.000$ & $0.385$ \\
        \textbf{GDiff c=0.9} & $0.805 \to 0.200$ & $0.992 \to 0.422$ & $0.205 \to 0.308$ \\
        \textbf{GDiff c=0.5} & $0.027 \to 0.001$ & $0.308 \to 0.032$ & $0.216 \to 0.297$ \\
        \textbf{GDiff c=0.1} & $0.000 \to 0.000$ & $0.000 \to 0.000$ & $0.221 \to 0.229$ \\
        \textbf{NPO} &  $0.000 \to 0.000$ & $0.000 \to 0.000$ & $0.179 \to 0.223$ \\
        \textbf{LossNorm} &  $0.432 \to 0.231$ & $0.752 \to 0.725$ & $0.214 \to 0.336$ \\
        \textbf{NGDiff} & $0.024 \to 0.012$ & $0.607 \to 0.747$  & $0.289 \to 0.353$ \\
        \bottomrule
    \end{tabular}
    \caption{Results of \textit{Verbmem}, \textit{Utility}, and \textit{TruthRatio} using different unlearning methods with AutoLR on the Phi-1.5 model. AutoLR improves the \textit{TruthRatio} and reduces \textit{Verbmem} across all methods. With or without AutoLR, NGDiff can significantly outperform other baseline methods.}
    \label{table:ablation_ngdiff}
\end{table}
\fi

\noindent\textbf{Impact of automatic learning rate}.
To evaluate the impact of AutoLR scheduler, we see in \Cref{table:ablation_ngdiff} all methods exhibit an increase in the \textit{TruthRatio} metric and a decrease in \textit{Verbmem}, though with some loss in the \textit{Utility}. For instance, \textit{LossNorm} benefits significantly from AutoLR with $\approx 20\%$ decrease in \textit{Verbmem}, and NGDiff increases its retaining \textit{Utility} and \textit{TruthRatio} by $>22\%$. We specifically demonstrate the impact of AutoLR on NGDiff in \Cref{fig:ablation_autolr}. Without AutoLR, the model's performance is highly sensitive to the static learning rates: when $\eta=10^{-5}$, the model fails to unlearn $\F$ as indicated by the low loss and high ROUGE score; in contrast, when $\eta=10^{-4}$, there is a significant drop in ROUGE score on the retain data, falling from 100\% to around 50\%. However, with the AutoLR scheduler, we observe a steady reduction in the \textit{Verbmem} (with the ROUGE forget close to 0 at convergence) while maintaining high utility (the ROUGE retain is 0.747, which is 19.5\% higher than the best results without AutoLR). 
% In addition, convergence with AutoLR is faster and more stable compared to using a static learning rate.

%\vspace{-0.2cm}
\section{Conclusion and Discussion}
We formulated the machine unlearning problem as a two-task optimization problem and proposed a novel unlearning method \algo based on normalized gradient difference and automatic learning rate adaption. By leveraging insights from multi-task optimization, \algo empirically improves forgetting quality while maintaining utility. We hope this paper helps establish a connection between LLM unlearning and multi-task optimization, and inspires further advancements in this field.
%, which can be provably (under Pareto optimality) and effectively solved by our novel unlearning method \textit{NGDiff}. 
%We also adapt GeN to set an automatic and adaptive learning rate scheduler. 

%ARR requires to have a limitation section. 
\section*{Limitations}

Like other machine learning approaches in NLP, while our goal is to remove the influence of specific documents from LLMs, complete removal cannot always be guaranteed. Therefore, caution should be exercised when applying the proposed unlearning techniques in practical applications as unlearned LLMs can still potentially generate harmful or undesired outputs. 

There are several technical alternatives that we did not explore in this paper due to its scope and limited resources. For example, other learning-rate-free methods could potentially be adapted as alternatives to the GeN approach used in this work. Additionally, other multi-task optimization methods could be applied to machine unlearning. However, scaling these approaches to the level of LLMs could be challenging, and are left as future work.

Finally, we mainly examined \algo's effectiveness on LLM unlearning in this paper with two benchmark datasets, TOFU and MUSE.
To show its generalizability, we provide an additional example to apply the algorithm to computer vision tasks (see Appendix \ref{sec:app_cv_exp}). 
However, it would be desirable to test NGDiff on other modalities beyond NLP and CV applications. 

% \section*{Ethical Consideration}
% We acknowledge that our work is aligned with the \textit{ACL Code of the Ethics} \footnote{\url{https://www.aclweb.org/portal/content/acl-code-ethics}} and will not raise ethical concerns.
% We do not use sensitive datasets/models that may cause any potential issues/risks.

\subsubsection*{Acknowledgments}
We would like to thank Shankar Ananthakrishnan, Fabian Triefenbach and Jianhua Lu from Amazon AGI Foundations team for providing feedback on this paper and supporting this work.

% \subsubsection*{Acknowledgments}
% Use unnumbered third level headings for the acknowledgments. All
% acknowledgments, including those to funding agencies, go at the end of the paper.

\bibliography{ref}

\begin{thebibliography}{10}

\bibitem{achiam2023gpt}
Josh Achiam, Steven Adler, Sandhini Agarwal, Lama Ahmad, Ilge Akkaya,
  Florencia~Leoni Aleman, Diogo Almeida, Janko Altenschmidt, Sam Altman,
  Shyamal Anadkat, et~al.
\newblock Gpt-4 technical report.
\newblock {\em arXiv preprint arXiv:2303.08774}, 2023.

\bibitem{bourtoule2021machine}
Lucas Bourtoule, Varun Chandrasekaran, Christopher~A Choquette-Choo, Hengrui
  Jia, Adelin Travers, Baiwu Zhang, David Lie, and Nicolas Papernot.
\newblock Machine unlearning.
\newblock In {\em 2021 IEEE Symposium on Security and Privacy (SP)}, pages
  141--159. IEEE, 2021.

\bibitem{boyd2004convex}
Stephen Boyd and Lieven Vandenberghe.
\newblock {\em Convex optimization}.
\newblock Cambridge university press, 2004.

\bibitem{bu2024automatic}
Zhiqi Bu and Shiyun Xu.
\newblock Gradient descent with generalized newton’s method.
\newblock In {\em The Thirteenth International Conference on Learning
  Representations}, 2024.

\bibitem{carlini2021extracting}
Nicholas Carlini, Florian Tramer, Eric Wallace, Matthew Jagielski, Ariel
  Herbert-Voss, Katherine Lee, Adam Roberts, Tom Brown, Dawn Song, Ulfar
  Erlingsson, et~al.
\newblock Extracting training data from large language models.
\newblock In {\em 30th USENIX Security Symposium (USENIX Security 21)}, pages
  2633--2650, 2021.

\bibitem{chen-yang-2023-unlearn}
Jiaao Chen and Diyi Yang.
\newblock Unlearn what you want to forget: Efficient unlearning for {LLM}s.
\newblock In Houda Bouamor, Juan Pino, and Kalika Bali, editors, {\em
  Proceedings of the 2023 Conference on Empirical Methods in Natural Language
  Processing}, pages 12041--12052, Singapore, December 2023. Association for
  Computational Linguistics.

\bibitem{chen2024machine}
Kongyang Chen, Zixin Wang, Bing Mi, Waixi Liu, Shaowei Wang, Xiaojun Ren, and
  Jiaxing Shen.
\newblock Machine unlearning in large language models.
\newblock {\em arXiv preprint arXiv:2404.16841}, 2024.

\bibitem{Chen2023OnTP}
Liang Chen, Shuming Ma, Dongdong Zhang, Furu Wei, and Baobao Chang.
\newblock On the pareto front of multilingual neural machine translation.
\newblock {\em ArXiv}, abs/2304.03216, 2023.

\bibitem{chen2021multi}
Shijie Chen, Yu~Zhang, and Qiang Yang.
\newblock Multi-task learning in natural language processing: An overview.
\newblock {\em ACM Computing Surveys}, 2021.

\bibitem{chen2018gradnorm}
Zhao Chen, Vijay Badrinarayanan, Chen-Yu Lee, and Andrew Rabinovich.
\newblock Gradnorm: Gradient normalization for adaptive loss balancing in deep
  multitask networks.
\newblock In {\em International conference on machine learning}, pages
  794--803. PMLR, 2018.

\bibitem{chen2020justpicksignoptimizing}
Zhao Chen, Jiquan Ngiam, Yanping Huang, Thang Luong, Henrik Kretzschmar, Yuning
  Chai, and Dragomir Anguelov.
\newblock Just pick a sign: Optimizing deep multitask models with gradient sign
  dropout, 2020.

\bibitem{de-cao-etal-2021-editing}
Nicola De~Cao, Wilker Aziz, and Ivan Titov.
\newblock Editing factual knowledge in language models.
\newblock In Marie-Francine Moens, Xuanjing Huang, Lucia Specia, and Scott
  Wen-tau Yih, editors, {\em Proceedings of the 2021 Conference on Empirical
  Methods in Natural Language Processing}, pages 6491--6506, Online and Punta
  Cana, Dominican Republic, November 2021. Association for Computational
  Linguistics.

\bibitem{defazio2023learningratefreelearningdadaptation}
Aaron Defazio and Konstantin Mishchenko.
\newblock Learning-rate-free learning by d-adaptation, 2023.

\bibitem{DEMP1977}
A.~P. Dempster, N.~M. Laird, and D.~B. Rubin.
\newblock Maximum likelihood from incomplete data via the {EM} algorithm.
\newblock {\em Journal of the Royal Statistical Society: Series B}, 39:1--38,
  1977.

\bibitem{desideri2012multiple}
Jean-Antoine D{\'e}sid{\'e}ri.
\newblock Multiple-gradient descent algorithm (mgda) for multiobjective
  optimization.
\newblock {\em Comptes Rendus Mathematique}, 350(5-6):313--318, 2012.

\bibitem{eldan2023whos}
Ronen Eldan and Mark Russinovich.
\newblock Who's harry potter? approximate unlearning in llms.
\newblock 2023.

\bibitem{gupta2021adaptive}
Varun Gupta, Christopher Jung, Seth Neel, Aaron Roth, Saeed Sharifi-Malvajerdi,
  and Chris Waites.
\newblock Adaptive machine unlearning.
\newblock {\em Advances in Neural Information Processing Systems},
  34:16319--16330, 2021.

\bibitem{han-etal-2023-fair}
Xudong Han, Timothy Baldwin, and Trevor Cohn.
\newblock Fair enough: Standardizing evaluation and model selection for
  fairness research in {NLP}.
\newblock In Andreas Vlachos and Isabelle Augenstein, editors, {\em Proceedings
  of the 17th Conference of the European Chapter of the Association for
  Computational Linguistics}, pages 297--312, Dubrovnik, Croatia, May 2023.
  Association for Computational Linguistics.

\bibitem{he2015deepresiduallearningimage}
Kaiming He, Xiangyu Zhang, Shaoqing Ren, and Jian Sun.
\newblock Deep residual learning for image recognition, 2015.

\bibitem{hulora}
Edward~J Hu, Phillip Wallis, Zeyuan Allen-Zhu, Yuanzhi Li, Shean Wang, Lu~Wang,
  Weizhu Chen, et~al.
\newblock Lora: Low-rank adaptation of large language models.
\newblock In {\em International Conference on Learning Representations}, 2021.

\bibitem{ivgi2023dogsgdsbestfriend}
Maor Ivgi, Oliver Hinder, and Yair Carmon.
\newblock Dog is sgd's best friend: A parameter-free dynamic step size
  schedule, 2023.

\bibitem{jang-etal-2023-knowledge}
Joel Jang, Dongkeun Yoon, Sohee Yang, Sungmin Cha, Moontae Lee, Lajanugen
  Logeswaran, and Minjoon Seo.
\newblock Knowledge unlearning for mitigating privacy risks in language models.
\newblock In Anna Rogers, Jordan Boyd-Graber, and Naoaki Okazaki, editors, {\em
  Proceedings of the 61st Annual Meeting of the Association for Computational
  Linguistics (Volume 1: Long Papers)}, pages 14389--14408, Toronto, Canada,
  July 2023. Association for Computational Linguistics.

\bibitem{javaloy2022rotogradgradienthomogenizationmultitask}
Adrián Javaloy and Isabel Valera.
\newblock Rotograd: Gradient homogenization in multitask learning, 2022.

\bibitem{jiang2023mistral}
Albert~Q Jiang, Alexandre Sablayrolles, Arthur Mensch, Chris Bamford,
  Devendra~Singh Chaplot, Diego de~las Casas, Florian Bressand, Gianna Lengyel,
  Guillaume Lample, Lucile Saulnier, et~al.
\newblock Mistral 7b.
\newblock {\em arXiv preprint arXiv:2310.06825}, 2023.

\bibitem{khaled2024dowgunleashedefficientuniversal}
Ahmed Khaled, Konstantin Mishchenko, and Chi Jin.
\newblock Dowg unleashed: An efficient universal parameter-free gradient
  descent method, 2024.

\bibitem{krizhevsky2009learning}
Alex Krizhevsky, Geoffrey Hinton, et~al.
\newblock Learning multiple layers of features from tiny images.
\newblock 2009.

\bibitem{li2024backdoorllm}
Yige Li, Hanxun Huang, Yunhan Zhao, Xingjun Ma, and Jun Sun.
\newblock Backdoorllm: A comprehensive benchmark for backdoor attacks on large
  language models, 2024.

\bibitem{li2023textbooks}
Yuanzhi Li, S{\'e}bastien Bubeck, Ronen Eldan, Allie Del~Giorno, Suriya
  Gunasekar, and Yin~Tat Lee.
\newblock Textbooks are all you need ii: phi-1.5 technical report.
\newblock {\em arXiv preprint arXiv:2309.05463}, 2023.

\bibitem{li2023avoiding}
Yucheng Li, Frank Geurin, and Chenghua Lin.
\newblock Avoiding data contamination in language model evaluation: Dynamic
  test construction with latest materials.
\newblock {\em arXiv preprint arXiv:2312.12343}, 2023.

\bibitem{linreasonable}
Baijiong Lin, Feiyang Ye, Yu~Zhang, and Ivor Tsang.
\newblock Reasonable effectiveness of random weighting: A litmus test for
  multi-task learning.
\newblock {\em Transactions on Machine Learning Research}, 2021.

\bibitem{liu2022continual}
Bo~Liu, Qiang Liu, and Peter Stone.
\newblock Continual learning and private unlearning.
\newblock In {\em Conference on Lifelong Learning Agents}, pages 243--254.
  PMLR, 2022.

\bibitem{liu2024conflictaversegradientdescentmultitask}
Bo~Liu, Xingchao Liu, Xiaojie Jin, Peter Stone, and Qiang Liu.
\newblock Conflict-averse gradient descent for multi-task learning, 2024.

\bibitem{liu2021towards}
Liyang Liu, Yi~Li, Zhanghui Kuang, J~Xue, Yimin Chen, Wenming Yang, Qingmin
  Liao, and Wayne Zhang.
\newblock Towards impartial multi-task learning.
\newblock In {\em iclr}, 2021.

\bibitem{liu2024rethinking}
Sijia Liu, Yuanshun Yao, Jinghan Jia, Stephen Casper, Nathalie Baracaldo, Peter
  Hase, Xiaojun Xu, Yuguang Yao, Hang Li, Kush~R Varshney, et~al.
\newblock Rethinking machine unlearning for large language models.
\newblock {\em arXiv preprint arXiv:2402.08787}, 2024.

\bibitem{liu2024rethinkingmachineunlearninglarge}
Sijia Liu, Yuanshun Yao, Jinghan Jia, Stephen Casper, Nathalie Baracaldo, Peter
  Hase, Yuguang Yao, Chris~Yuhao Liu, Xiaojun Xu, Hang Li, Kush~R. Varshney,
  Mohit Bansal, Sanmi Koyejo, and Yang Liu.
\newblock Rethinking machine unlearning for large language models, 2024.

\bibitem{maini2024tofu}
Pratyush Maini, Zhili Feng, Avi Schwarzschild, Zachary~C Lipton, and J~Zico
  Kolter.
\newblock Tofu: A task of fictitious unlearning for llms.
\newblock {\em arXiv preprint arXiv:2401.06121}, 2024.

\bibitem{mishchenko2024prodigyexpeditiouslyadaptiveparameterfree}
Konstantin Mishchenko and Aaron Defazio.
\newblock Prodigy: An expeditiously adaptive parameter-free learner, 2024.

\bibitem{ouyang2022traininglanguagemodelsfollow}
Long Ouyang, Jeff Wu, Xu~Jiang, Diogo Almeida, Carroll~L. Wainwright, Pamela
  Mishkin, Chong Zhang, Sandhini Agarwal, Katarina Slama, Alex Ray, John
  Schulman, Jacob Hilton, Fraser Kelton, Luke Miller, Maddie Simens, Amanda
  Askell, Peter Welinder, Paul Christiano, Jan Leike, and Ryan Lowe.
\newblock Training language models to follow instructions with human feedback,
  2022.

\bibitem{pawelczyk2023context}
Martin Pawelczyk, Seth Neel, and Himabindu Lakkaraju.
\newblock In-context unlearning: Language models as few shot unlearners.
\newblock In {\em ICML}, 2024.

\bibitem{penedo2023refinedwebdatasetfalconllm}
Guilherme Penedo, Quentin Malartic, Daniel Hesslow, Ruxandra Cojocaru,
  Alessandro Cappelli, Hamza Alobeidli, Baptiste Pannier, Ebtesam Almazrouei,
  and Julien Launay.
\newblock The refinedweb dataset for falcon llm: Outperforming curated corpora
  with web data, and web data only, 2023.

\bibitem{radford2019language}
Alec Radford, Jeffrey Wu, Rewon Child, David Luan, Dario Amodei, Ilya
  Sutskever, et~al.
\newblock Language models are unsupervised multitask learners.
\newblock {\em OpenAI blog}, 1(8):9, 2019.

\bibitem{samdani-etal-2012-unified}
Rajhans Samdani, Ming-Wei Chang, and Dan Roth.
\newblock Unified expectation maximization.
\newblock In Eric Fosler-Lussier, Ellen Riloff, and Srinivas Bangalore,
  editors, {\em Proceedings of the 2012 Conference of the North {A}merican
  Chapter of the Association for Computational Linguistics: Human Language
  Technologies}, pages 688--698, Montr{\'e}al, Canada, June 2012. Association
  for Computational Linguistics.

\bibitem{shi2024muse}
Weijia Shi, Jaechan Lee, Yangsibo Huang, Sadhika Malladi, Jieyu Zhao, Ari
  Holtzman, Daogao Liu, Luke Zettlemoyer, Noah~A Smith, and Chiyuan Zhang.
\newblock Muse: Machine unlearning six-way evaluation for language models.
\newblock {\em arXiv preprint arXiv:2407.06460}, 2024.

\bibitem{si2023knowledgeunlearningllmstasks}
Nianwen Si, Hao Zhang, Heyu Chang, Wenlin Zhang, Dan Qu, and Weiqiang Zhang.
\newblock Knowledge unlearning for llms: Tasks, methods, and challenges, 2023.

\bibitem{exhibitj}
The New~York Times.
\newblock Exhibit j, 2023.

\bibitem{touvron2023llama}
Hugo Touvron, Louis Martin, Kevin Stone, Peter Albert, Amjad Almahairi, Yasmine
  Babaei, Nikolay Bashlykov, Soumya Batra, Prajjwal Bhargava, Shruti Bhosale,
  et~al.
\newblock Llama 2: Open foundation and fine-tuned chat models.
\newblock {\em arXiv preprint arXiv:2307.09288}, 2023.

\bibitem{ullah2021machine}
Enayat Ullah, Tung Mai, Anup Rao, Ryan~A Rossi, and Raman Arora.
\newblock Machine unlearning via algorithmic stability.
\newblock In {\em Conference on Learning Theory}, pages 4126--4142. PMLR, 2021.

\bibitem{wallace-etal-2021-concealed}
Eric Wallace, Tony Zhao, Shi Feng, and Sameer Singh.
\newblock Concealed data poisoning attacks on {NLP} models.
\newblock In Kristina Toutanova, Anna Rumshisky, Luke Zettlemoyer, Dilek
  Hakkani-Tur, Iz~Beltagy, Steven Bethard, Ryan Cotterell, Tanmoy Chakraborty,
  and Yichao Zhou, editors, {\em Proceedings of the 2021 Conference of the
  North American Chapter of the Association for Computational Linguistics:
  Human Language Technologies}, pages 139--150, Online, June 2021. Association
  for Computational Linguistics.

\bibitem{wang2020gradientvaccineinvestigatingimproving}
Zirui Wang, Yulia Tsvetkov, Orhan Firat, and Yuan Cao.
\newblock Gradient vaccine: Investigating and improving multi-task optimization
  in massively multilingual models, 2020.

\bibitem{wu2024akewassessingknowledgeediting}
Xiaobao Wu, Liangming Pan, William~Yang Wang, and Anh~Tuan Luu.
\newblock Akew: Assessing knowledge editing in the wild, 2024.

\bibitem{xin2022current}
Derrick Xin, Behrooz Ghorbani, Justin Gilmer, Ankush Garg, and Orhan Firat.
\newblock Do current multi-task optimization methods in deep learning even
  help?
\newblock {\em Advances in neural information processing systems},
  35:13597--13609, 2022.

\bibitem{yu2020gradient}
Tianhe Yu, Saurabh Kumar, Abhishek Gupta, Sergey Levine, Karol Hausman, and
  Chelsea Finn.
\newblock Gradient surgery for multi-task learning.
\newblock {\em Advances in Neural Information Processing Systems},
  33:5824--5836, 2020.

\bibitem{NIPS2001_a0128693}
Alan~L Yuille and Anand Rangarajan.
\newblock The concave-convex procedure (cccp).
\newblock In T.~Dietterich, S.~Becker, and Z.~Ghahramani, editors, {\em
  Advances in Neural Information Processing Systems}, volume~14. MIT Press,
  2001.

\bibitem{zhang2024rightforgotteneralarge}
Dawen Zhang, Pamela Finckenberg-Broman, Thong Hoang, Shidong Pan, Zhenchang
  Xing, Mark Staples, and Xiwei Xu.
\newblock Right to be forgotten in the era of large language models:
  Implications, challenges, and solutions, 2024.

\bibitem{zhang2024negative}
Ruiqi Zhang, Licong Lin, Yu~Bai, and Song Mei.
\newblock Negative preference optimization: From catastrophic collapse to
  effective unlearning.
\newblock {\em arXiv preprint arXiv:2404.05868}, 2024.

\bibitem{zhang-etal-2023-survey}
Zhihan Zhang, Wenhao Yu, Mengxia Yu, Zhichun Guo, and Meng Jiang.
\newblock A survey of multi-task learning in natural language processing:
  Regarding task relatedness and training methods.
\newblock In Andreas Vlachos and Isabelle Augenstein, editors, {\em Proceedings
  of the 17th Conference of the European Chapter of the Association for
  Computational Linguistics}, pages 943--956, Dubrovnik, Croatia, May 2023.
  Association for Computational Linguistics.

\end{thebibliography}
\bibliographystyle{plain}

%\appendix
%\bibliography{iclr2025_conference}
% \bibliographystyle{acl_natbib}

\clearpage
\appendix

\section{Preliminary evidence}\label{sec:preliminary_evidence}
In our preliminary experiments, we observe two key issues preventing the standard methods from being practically applied. First, balancing retaining and forgetting losses is difficult. In \Cref{fig:intro-example}, we observe a trade-off between the performance on $\R$ and $\F$, where some methods fail to unlearn $\F$ (points in the upper-right corner of the left figure), and some do not maintain utility in $\R$ (points in the bottom-left corner of the left figure). The  {\color{blue}blue dotted} line in \Cref{fig:intro-example} further illustrates the trade-off in GDiff by sweeping a hyper-parameter $c\in [0,1],$ which is used to balance the losses on the forgetting and retaining data (see Eq. \eqref{eq:static}). Picking an appropriate $c$ to balance the two terms is often challenging. Secondly, the optimization methods for unlearning are usually sensitive to the learning rate. As illustrated in \Cref{fig:grad_diff_lr_rouge}, even for the same algorithm, various learning rates lead to substantial changes in the ROUGE scores and loss values, making the unlearning methods unstable and difficult to use in practice.

\begin{figure*}[!h]
\centering
\includegraphics[width=0.99\linewidth]{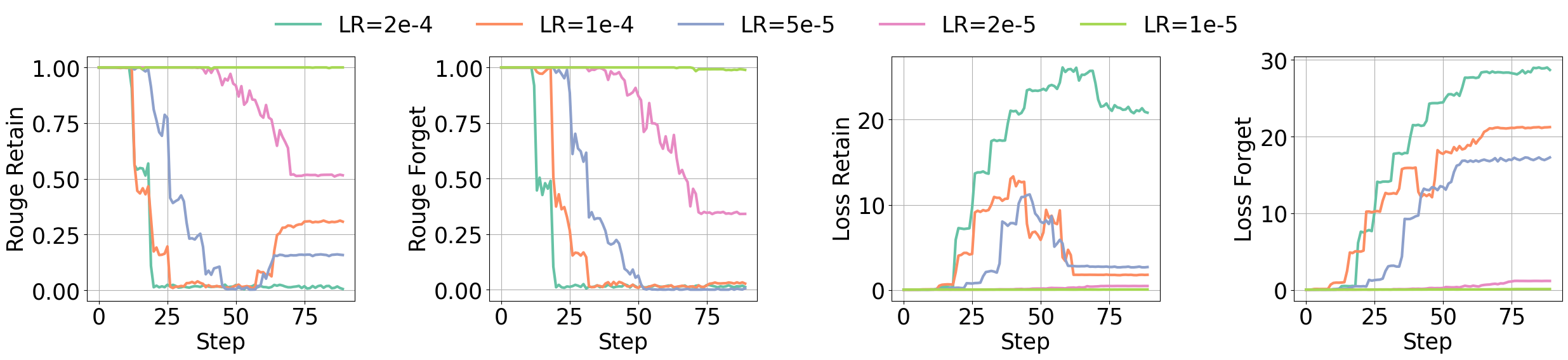}
\caption{ROUGE scores and loss values during unlearning with vanilla GDiff (equally weighted), under different learning rates to which the unlearning performance is highly sensitive.
% [multiple GD weights, retain rouge v.s. iterations]. A-Right: GD is sensitive to weights [multiple GD weights, forget rouge v.s. iterations].
% \\
% B-Left: GD is sensitive to lr [one GD (at the turning point of above figure, guess 0.5?), multiple lr, retain rouge v.s. iterations]. B-Right: GD is sensitive to lr [one GD (at the turning point of above figure, guess 0.5?), multiple lr, forget rouge v.s. iterations]
% \\
% Can we plot all plots again using loss instead of rouge (total 8 sub-plots)?
}
\label{fig:grad_diff_lr_rouge}
\end{figure*}

\section{Comparing NGDiff with GradNorm}
\label{app:gradnorm}
\begin{algorithm}[!htb]
\small
\caption{NGDiff}
\begin{algorithmic}[1]
\For{$t=1,2,...$}
\State Compute retaining loss $L_\R(\bm\theta_t)$ by one forward pass
\State Compute retaining gradient $\g_\R(\bm\theta_t)=\nabla_{\bm\theta}{L_\R}$
\State Compute forgetting loss $L_\F(\bm\theta_t)$ by one forward pass
\State Compute forgetting gradient $\g_\F(\bm\theta_t)=\nabla_{\bm\theta}{L_\F}$
\State Construct unlearning gradient $\g_\textup{NGDiff}=\g_\R/\|\g_\R\|-\g_\F/\|\g_\F\|$
\State Update $\bm\theta_{t+1}=\bm\theta_t-\eta \g_\textup{NGDiff}$
\EndFor
\end{algorithmic}
\label{alg:appendix_ngdiff}
\end{algorithm}

\begin{algorithm}[!htb]
\small
\caption{GradNorm for two-task}
\begin{algorithmic}[1]
\State Initialize the scalaring coefficients $w_\R(\bm\theta_0) = 1$ and $w_\F(\bm\theta_0) = 1$
\State Pick value for $\alpha > 0$ and pick the weights $\bm\theta_\text{LS}$ (the last shared layer of $\bm\theta_t$)

\For{$t=1,2,...$}
\State {Compute retaining loss $L_\R(\bm\theta_t)$ by one forward pass}
\State \textcolor{red}{Compute retaining gradient $\g_\R(\bm\theta_\text{LS})=\nabla_{\bm\theta_\text{LS}}{L_\R}$}
\State {Compute forgetting loss $L_\F(\bm\theta_t)$ by one forward pass}
\State \textcolor{red}{Compute forgetting gradient $\g_\F(\bm\theta_\text{LS})=\nabla_{\bm\theta_\text{LS}}{L_\F}$ }
\State Compute loss $L(\bm\theta_t) = w_\R(\bm\theta_t)L_\R(\bm\theta_t) + w_\F(\bm\theta_t)L_\F(\bm\theta_t)$
\State {\color{red}Compute $\Bar{\g}(\bm\theta_\text{LS})$ by averaging $\g_\R$ and $\g_\F$}
\State {\color{red}Compute GradNorm loss
$$L_{GN}(\bm\theta_t) = |\g_\R - \Bar{\g} \times [r_\R(t)]^\alpha|_1 + |\g_\F - \Bar{\g} \times [r_\F(t)]^\alpha|_1$$}
\State {\color{red}Compute GradNorm gradients $\nabla_{w_\R} L_{GN}$ and $\nabla_{w_\F} L_{GN}\in\R$}
\State Compute the full gradient $\nabla_{\bm\theta_t} L$
\State {\color{red}Update $w_\R(\bm\theta_{t}) \to w_\R(\bm\theta_{t+1})$ and $w_\F(\bm\theta_{t}) \to w_\F(\bm\theta_{t+1})$ using $\nabla_{w_\R} L_{GN}$ and $\nabla_{w_\F} L_{GN}$}
\State Update $\bm\theta_{t+1}=\bm\theta_t-\eta \nabla_{\bm\theta_t} L$
\State Renormalize $w_\R(\bm\theta_{t+1})$ and $w_\F(\bm\theta_{t+1})$ so that $w_\R(\bm\theta_{t+1}) +  w_\F(\bm\theta_{t+1}) = 2$
\EndFor
\end{algorithmic}
\label{alg:appendix_gradnorm}
\end{algorithm}

We compare the GradNorm algorithm \cite{chen2018gradnorm} with our proposed method, NGDiff. We highlight some steps of GradNorm in \textcolor{red}{red} to indicate the differences than NGDiff:
% GradNorm first computes the loss values and gradients on the retaining and forgetting sets separately, then combines the two losses into a weighted model loss $L(\bm\theta_t)$ by where the task weights are learnable. The model parameter are then updated based on the total loss and the task weights relative to the task gradients.

\begin{itemize}
\item NGDiff sets the scalaring coefficient as $1/\|\g_\R\|$ and $1/\|\g_\F\|$, while GradNorm uses gradient descent to learn these coefficients as $w_\R$ and $w_\F$.
    \item NGDiff is model-agnostic while GradNorm contains specific designs for multi-task architecture. In unlearning, there are \textbf{2 data splits} (i.e., $\F$ and $\R$) and each data split defines one task. Hence all model parameters are shared. However, in the original form of GradNorm, there is \textbf{1 data split} on which multiple tasks are defined (can be more than 2). Hence the model parameters are partitioned into [shared layers, task 1 specific layers, task 2 specific layers].
    \item NGDiff computes the full per-task gradients whereas GradNorm only computes the last shared layer's gradients.
    \item NGDiff requires 2 back-propagation at each iteration but GradNorm requires 3 (2 for per-task gradients, 1 for $\nabla_{\bm\theta}L$), which may translate to more training time for large models.
    \item GradNorm introduces additional hyperparameters that can be difficult and costly to tune, and may cause instability of training if not properly tuned. These hyperparameters include $\alpha$ and two learning rates to update $w_\R$ and $w_\F$ in Line 13 of \Cref{app:gradnorm}. In contrast, NGDiff is hyperparameter-free when equipped with GeN (AutoLR).
    \item NGDiff are theoretically supported by \Cref{thm:NGDiff}, while the choice of hyperparameters and the use of a heuristic $r_i(t)$ by GradNorm may require further justification. Here $r_i(t)=\widetilde{L}_i(\bm\theta_t) / \E_{\text{task}}[\widetilde{L}_i(\bm\theta_t)]$ is the "relative inverse training rate" of task $i$, where $\widetilde{L}_i(\bm\theta_t) = L_i(\bm\theta_t) / L_i(\bm\theta_0)$, $i \in \{\F, \R\}$.
\end{itemize}
In summary, NGDiff is remarkably simpler and more well-suited than GradNorm for unlearning, with stable performance and theoretical ground.

\section{Details related to GeN}
\label{app:GeN}
\subsection{Brief introduction of GeN}
GeN \citep{bu2024automatic} is a method that sets the learning rate for any given gradient $\d$ as
$$\eta_\text{GeN}=\frac{\G^\top \d}{\d^\top \H\d}$$
where $\G$ is the gradient and $\H$ is the Hessian matrix of some loss $L$. One only needs to access the scalars $\G^\top\d$ and $\d^\top\H\d$, without computing the high-dimensional $\G$ and $\H$ (or Hessian-vector product). To do so, two additional forward passes are needed: given a constant (say $\xi=0.001$), we compute $L(\bm\theta+\xi\d)$ and $L(\bm\theta-\xi\d)$. Then by curve fitting or finite difference as demonstrated below, we can estimate up to arbitrary precision controlled by $\xi$:
$$\G^\top\d\approx\frac{L(\bm\theta+\xi\d)-L(\bm\theta-\xi\d)}{2\xi}$$
and
$$\d^\top\H\d\approx\frac{L(\bm\theta+\xi\d)-2L(\bm\theta)+L(\bm\theta-\xi\d)}{\xi^2}$$

Notice that the regular optimization requires 1 forward pass and 1 back-propagation; GeN requires in total 3 forward passes and 1 back-propagation. Given that back-propagation costs roughly twice the computation time than forward pass, the total time increases from 3 units of time to 5 units. Nevertheless, GeN needs not to be applied at each iteration: if we update the learning rate every 10 iterations as in \Cref{rem:GeN lazy}, the total time reduces to $3+2/10=3.2$ units, and the overhead is less than 10\% compared to the regular optimization.

\subsection{Adapting GeN to unlearning}
Naively applying GeN to the unlearning will result in
$$\eta_\text{GeN}=\frac{\G^\top\g_\text{UN}}{\g_\text{UN}^\top \H\g_\text{UN}}$$
which minimizes the loss over all datapoints, in both $\F$ and $\R$. This is against our goal to maximize the forgetting loss. We must consider the learning rate separately for $\F$ and $\R$, as shown in \Cref{app:proofs} (Proof of \Cref{thm:NGDiff}). When both losses have a convex curvature in \Cref{fig:parabola}, the optimal learning rate is only well-defined for $L_\R$ and we do not claim to maximize $L_\F$. In other words, if we minimize $L_\R$, we get to worsen $L_\F$ (though not maximally); if we choose to maximize $L_\F$, we will use infinite learning rate that also maximizes $L_\R$. Therefore, our learning rate in \eqref{eq:optim lr} only uses $\R$ instead of the whole dataset.

\section{Computer Vision Experiments}
\label{sec:app_cv_exp}
\begin{table*}[!htb]
    \centering
    \small
    \setlength{\tabcolsep}{6pt}
    \begin{tabular}{c|cc|cc}
        \toprule
        \multirow{2}{*}{\textbf{Method}} & \multicolumn{2}{c|}{\textbf{CIFAR-10}} & \multicolumn{2}{c}{\textbf{CIFAR-100}} \\
	& \textit{Forget Acc} $\downarrow$  & \textit{Retain Acc} $\uparrow$ & \textit{Forget Acc} $\downarrow$  & \textit{Retain Acc} $\uparrow$\\
        \midrule
        \textbf{No-unlearn} & $0.926$ & $0.956$ & $0.745$ & $0.750$  \\\hline
        \textbf{GDiff c=0.9} & $0.000$ & $0.817$ & $0.000$ & $0.664$ \\
        \textbf{GDiff c=0.5} & $0.000$ & $0.830$ & $0.000$ & $0.609$ \\
        \textbf{GDiff c=0.1} & $0.000$ & $0.825$ & $0.000$ & $0.667$ \\
        \textbf{LossNorm} & $0.000$ & $0.753$ & $0.000$ & $0.432$ \\
        \textbf{NGDiff} & $0.000$ & $\textbf{0.931}$ & $0.000$ & $\textbf{0.701}$ \\
        \bottomrule
    \end{tabular}
    \caption{Results of \textit{Forget Acc} and \textit{Retain Acc} using different unlearning methods on the CIFAR-10 dataset. Compared to other baseline methods, \textit{NGDiff} has the best performance on the model utility.}
    \label{table:appendix_cv_experiments}
\end{table*}

To demonstrate the effectiveness of unlearning  across other modalities, we also evaluate our method on the image classification task. Specifically, we choose the CIFAR-10 and CIFAR-100 dataset \cite{krizhevsky2009learning} and train a ResNet-50 \cite{he2015deepresiduallearningimage} model from scratch. For the CIFAR-10 dataset, we sample 500 images from the class $dog$ as the forgetting data, and use images from the remaining 9 classes as the retaining data. For the CIFAR-100 dataset, we sample 500 images from the class $bed$ as the forgetting data, and use images from the remaining 99 classes as the retaining data. After training, the initial forget data accuracy is $0.926$, and the retain data accuracy is $0.956$ on the CIFAR-10 dataset. The initial forget data accuracy is $0.745$, and the retain data accuracy is $0.750$ on the CIFAR-100 dataset. Then we apply different unlearning methods to the trained models. As shown in Table \ref{table:appendix_cv_experiments}, all methods successfully reduce the forget accuracy to $0$. However, the retaining accuracy of \textit{NGDiff} remains the highest, which shows its effectiveness in preserving the model utility in image classification tasks.

\section{Other Related Works}
\label{app:related}

\paragraph{Machine unlearning}
Machine unlearning is oftentimes viewed as a continual learning approach, that removes specific data points after a model has been trained to memorize them. Such removal is light-weighted in contrast to re-training, especially when the forgetting set is much smaller than the retaining. In addition to the methods already introduced in \Cref{section:pareto} (namely GA, GDiff and NPO), other methods include SISA \cite{bourtoule2021machine}, influence functions \cite{ullah2021machine}, differential privacy \cite{gupta2021adaptive} and so on. However, these methods could be difficult to scale on large models and large datasets due to the algorithmic complexity. To our best knowledge, this is the first work that formulate the unlearning problem as a two-task problem, which can be solved by a number of well-known MTO methods.

\paragraph{Multi-task optimization}
MTO is a paradigm where one model is trained to perform multiple tasks simultaneously, so as to significantly improve the efficiency in contrast to training multiple models, one for each task. The key challenge of MTO is the performance trade-off among tasks, where the multi-task model is worse than single-task model if trained on each task separately. Therefore, the core idea is to balance different tasks by modifying the per-task gradients, e.g. with normalization (LossNorm and NGDiff), PCGrad \cite{yu2020gradient}, RLW \cite{linreasonable}, IMTL \cite{liu2021towards}, MGDA \cite{desideri2012multiple}, CAGrad \cite{liu2024conflictaversegradientdescentmultitask}, GradVaccine \cite{wang2020gradientvaccineinvestigatingimproving}, GradDrop \cite{chen2020justpicksignoptimizing}, RotoGrad \cite{javaloy2022rotogradgradienthomogenizationmultitask}, etc.

\paragraph{Learning-rate-free methods}
Parameter-free or learning-rate-free methods automatically set the learning rate scheduler without the hyperparameter tuning, which is computationally infeasible for LLMs, e.g. LLAMA2 pre-training uses 3 hyperparameters just for the learning rate: warmup steps, peak learning rate, and minimum learning rate. At high level, there are two approaches to learning-rate-free methods. 

On one hand, GeN \cite{bu2024automatic} leverages the Taylor expansion and convex-like landscape of deep learning, which is applicable for the general purpose, even if the gradient is modified like in the unlearning.

On the other hand, methods like D-adaptation \cite{defazio2023learningratefreelearningdadaptation}, Prodigy \cite{mishchenko2024prodigyexpeditiouslyadaptiveparameterfree}, DoG \cite{ivgi2023dogsgdsbestfriend}, DoWG  \cite{khaled2024dowgunleashedefficientuniversal} are based on the convex and $G$-Lipschitz conditions: $L(\bar\theta_T)-L(\bm\theta_*)\leq \frac{|\theta_0-\theta_*|^2}{2\eta T}+\frac{\eta G^2}{2}$ where $\theta_*$ is the unknown minimizer of $L$ and $\bar\theta_T$ is an averaging scheme of $\{\theta_0,...,\theta_T\}$. With the same theoretical foundation, these methods propose different ways to approximate the initial-to-final distance $|\theta_0-\theta_*|$. There are two main issues to apply these methods on the unlearning. Firstly, the assumption of $G$-Lipschitz is hard to verify and the minimizer $\theta_*$ is not well-defined in multi-objective (see our discussion on Pareto optimality under \Cref{thm:pareto}). Secondly, the optimal learning rate $\frac{|\theta_0-\theta_*|}{G\sqrt{T}}$ is defined in a manner to minimize the loss, whereas MTO methods operate on the gradient level. Hence MTO is incompatible to such parameter-free methods given that we cannot derive a corresponding loss (e.g. there exists no $L_\textup{NGDiff}$ such that $\frac{\partial L_\textup{NGDiff}}{\partial \theta}=\g_\textup{NGDiff}$).

\section{Experiments}
\label{sec:app_experiment}

\subsection{Datasets}

\label{sec:app_dataset}

\begin{table*}[!htb]
    \centering
    \small
    \setlength{\tabcolsep}{4pt}
    \begin{tabular}{c|ccc|ccccc}
        \toprule
        \multirow{2}{*}{\textbf{Dataset}} & \multicolumn{3}{c|}{\textbf{TOFU}} & \multicolumn{4}{c}{\textbf{MUSE-NEWS}} \\
	& Full & Forget10 & Retain90 & Train & Raw & Verbmem$_{F}$ & Knowmem$_{F}$ & Knowmem$_{R}$ \\
        \midrule
        \# samples & $4,000$ & $400$ & $3,600$ & $7,110$ & $2,669$ & $100$ & $100$ & $100$ \\
        \bottomrule
    \end{tabular}
    \caption{Statistics of the \textit{TOFU} and \textit{MUSE-NEWS} datasets. For the \textit{TOFU} dataset, we use \textit{Full} split for training the target model, \textit{Forget10} and \textit{Retain90} as the forgetting and retaining split for unlearning experiments. For the \textit{MUSE-NEWS} dataset, we utilize \textit{Train} split for training, \textit{Raw} split for unlearning. For evaluation, we use Verbmem$_{F}$ and Knowmem$_{F}$ splits from forgetting data, and Knowmem$_{R}$ split from the retaining data.}
    \label{table:dataset_stats}
\end{table*}
To evaluate the empirical performance of our proposed method, we experiment on the following datasets in \Cref{table:dataset_stats}.
\begin{itemize}
    \item \textit{Task of Fictitious Unlearning (TOFU)} \cite{maini2024tofu}. This dataset consists of question-answer pairs based on fictitious author biographies generated by GPT-4 \cite{achiam2023gpt}. Initially, predefined attributes, such as birthplace, gender, and writing genre, are assigned to 200 distinct authors. GPT-4 is then prompted to generate detailed information about each author. Following the synthesized data, 20 question-answer pairs are created for each fictitious author. The dataset is then divided into distinct datasets: the retaining set and the forgetting set. In our experiments, we use the \textit{forget10} and \textit{retain90} split, which excludes 10\% of the original dataset.
    \item  \textit{MUSE-NEWS} \cite{shi2024muse}. This dataset consists of BBC news articles \cite{li2023avoiding} from August 2023. It includes seven subsets of news data: \textit{raw}, \textit{verbmem}, \textit{knowmem}, \textit{privleak}, \textit{scal}, \textit{sust}, and \textit{train}. We utilize the \textit{train} split to finetune a target model, and then the \textit{raw} set, which includes both the forget and retain data, for the target model unlearning. Then, we use \textit{verbmem}, \textit{knowmem} split to evaluate the unlearned model's performance.
\end{itemize}

\subsection{Evaluation Metrics}

Following the existing work \cite{shi2024muse}, we evaluate the unlearning performance based on the quality of outputs from the model after unlearning. We expect a good performance should satisfy the following requirements:

\textbf{No verbatim memorization} 
We evaluate this metric by prompting the model with the first $l$ tokens of the news data in the forget set and compare the model's continuation outputs with the ground truth continuation. Specifically, for each input $x \in F$, we choose $x_{[:l]}$ as input, and compare the output $f(x_{[:l]})$ with the ground truth continuation $x_{[l+1:]}$ with the ROUGE-L recall score:

\begin{equation}
\begin{split}
&    \textit{Verbmem}(f, F) \\
&= \frac{1}{\|F\|} \Sigma_x \: \textsc{ROUGE-L}(f(x_{[:l]}), x_{[l+1:]})
    \end{split}
\end{equation}

\textbf{No knowledge memorization}
To evaluate this metric, we use the generated question-answer pair based on each example $x \in F$. We prompt the model with the question part $q$ and compare the output answer $f(q)$ to the ground truth answer $a$ using ROUGE-L recall scores:

\begin{equation}
    \textit{Knowmem}(f, F) = \frac{1}{\|F\|} \Sigma_x \: \textsc{ROUGE-L}(f(q), a)
\end{equation}

\textbf{Maintained model utility} An effective unlearning method should also maintain the model's performance on the retain data. For the \textit{MUSE-NEWS} dataset, we use the \textit{Knowmem$_r$} split, which consists of the generated question-answer pairs based on the retain data. For the \textit{TOFU} dataset, we prompt the model with the question from the retain set and compare the generated answer with the ground truth. We use ROUGE-L recall scores for evaluation:

\begin{equation}
    \textit{Utility}(f, R) = \frac{1}{\|R\|} \Sigma_x \: \textsc{ROUGE-L}(f(q), a)
\end{equation}

Additionally, we evaluate the model using the \textit{Retain10-perturbed} split from the \textit{TOFU} dataset. It consists of five perturbed answers for each original answer, keeping original template but modifying the facts. We compute the Truth Ratio metric, which compares the likelihood of the model generating a correct answer versus an incorrect one for each question in the retain set. A higher Truth Ratio indicates better model utility that effectively remembers knowledge from the retain data.

\subsection{Hyper-parameter Settings}
To finetune a targeted model with the full dataset, we use the optimizer Adam with a learning rate of $\eta=\{10^{-5}, 2*10^{-5}\}$, a training batch size of \{16, 32\}, and train 25 epochs for all language models. For the unlearning process, we use the optimizer Adam with a learning rate $\eta=\{10^{-5}, 5*10^{-5}, 10^{-4}\}$, and train 15 epochs for all unlearning methods. For the Phi-1.5, Falcon-1B, and GPT2-XL models, we perform full-model parameter tuning. For the Llama2-7B and Mistral-7B models, we apply the LoRA training method \cite{hulora} with rank = 8.

\subsection{Other unlearning results}
\begin{table*}[!htb]
    \centering
    \small
    \setlength{\tabcolsep}{4pt}
    \begin{tabular}{c|c|ccccccc}
        \toprule
        \multirow{2}{*}{\textbf{Base Model}} 
 &\multirow{2}{*}{\textbf{Metric}} & \multicolumn{7}{c}{\textbf{Method}}\\
	& & \textit{No-unlearn} & \textit{GDiff-0.9} & \textit{GDiff-0.5} & \textit{GDiff-0.1} & \textit{NPO} & \textit{LossNorm} & \textit{NGDiff} \\
        \midrule
         \multirow{3}{*}{Phi-1.5} & Verbmem $\downarrow$ & $0.018$ & $0.000$ & $0.012$ & $0.000$ & $0.000$ & $0.012$ & $\textbf{0.004}$ \\
        & Utility $\uparrow$ & $0.372$ & $0.277$ & $0.061$ & $0.000$ & $0.000$ & $0.061$ & $\textbf{0.001}$ \\
        & Knowmem $\downarrow$ & $0.030$ & $0.000$ & $0.002$ & $0.000$ & $0.000$ & $0.002$ & $\textbf{0.023}$ \\
         \midrule
         \multirow{3}{*}{Falcon-1B} & Verbmem $\downarrow$ & $0.204$ & $0.132$ & $0.000$ & $0.000$ & $0.000$ & $0.126$ & $\textbf{0.000}$ \\
        & Utility $\uparrow$ & $0.386$ & $0.214$ & $0.000$ & $0.000$ & $0.000$ & $0.142$ & $\textbf{0.025}$ \\
        & Knowmem $\downarrow$ & $0.232$ & $0.078$ & $0.000$ & $0.000$ & $0.000$ & $0.130$ & $\textbf{0.087}$ \\
        \bottomrule
    \end{tabular}
    \caption{Results of \textit{Verbatim memorization}, \textit{Model utility}, and \textit{TruthRatio} on \textit{MUSE-NEWS} dataset with different unlearning methods on Phi-1.5, and Falcon-1B models. Lower \textit{Verbmem} along with higher \textit{Utility} and \textit{TruthRatio} indicate a more superior unlearning strategy.}
    \label{table:results_unlearning_muse_appendix}
\end{table*}

% \subsection{Unused plots\textcolor{red}{ZB: change this}}
% \begin{figure}[!htb]
% \centering
% \includegraphics[width=0.24\linewidth]{figures/gd_rouge_retain_step.png}
% \includegraphics[width=0.24\linewidth]{figures/gd_rouge_forget_step.png}
% \includegraphics[width=0.24\linewidth]{figures/gd_rouge_retain_lr_step.png}
% \includegraphics[width=0.24\linewidth]{figures/gd_rouge_forget_lr_step.png}
% \includegraphics[width=0.24\linewidth]{figures/gd_loss_retain_step.png}
% \includegraphics[width=0.24\linewidth]{figures/gd_loss_forget_step.png}
% \includegraphics[width=0.24\linewidth]{figures/gd_loss_retain_lr_step.png}
% \includegraphics[width=0.24\linewidth]{figures/gd_loss_forget_lr_step.png}
% \caption{ROUGE scores and loss values on the forget and retain data during unlearning with Gradient Difference method. The left four figures illustrate the unlearning performance with various coefficients on the forget loss, while the right four figures show the performance with different learning rates. Results indicate that the unlearning performance is highly sensitive to both the forget coefficients and the learning rates.}
% \end{figure}

More results on \textit{MUSE-NEWS} dataset with Phi-1.5 model and Falcon-1B model are in Table \ref{table:results_unlearning_muse_appendix}.

\begin{figure}[t]
%{r}{0.6\linewidth}
% [!htb]
%\vspace{-0.9cm}
\centering
\includegraphics[width=0.34\linewidth]{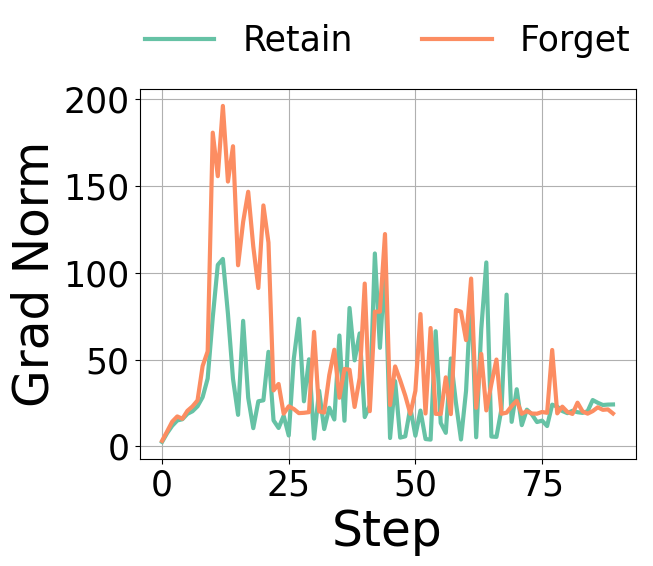}
\includegraphics[width=0.37\linewidth]{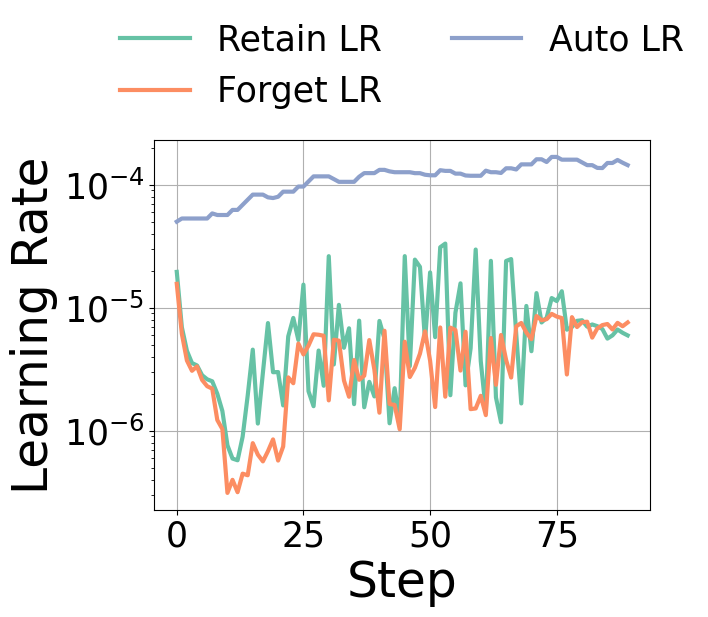}
% \caption{Monitoring grad norm (left: forget and retain) and learning rates (right: 1. auto lr, 2. auto lr/forget norm (forget lr), 3. retain lr) v.s. iterations}
%\vspace{-0.1cm}
\caption{Gradient norms and learning rates during the unlearning on \textit{TOFU} dataset using NGDiff and Phi-1.5 model. The AutoLR scheduler assigns a smaller learning rate to the forgetting gradient, effectively preserving model utility on the retaining set.}
\label{fig:monitor}
\end{figure}

\subsection{Visualization of learning rate scheduling} We monitor the gradient norms and the learning rate in \Cref{fig:monitor} for \Cref{alg:auto unlearning}. We observe that the automatic learning rate is indeed effective, picking up from $5e-5$ to around $2e-4$, and that NGDiff assigns a smaller learning rate to the forgetting gradient, not perturbing the model too much to maintain the high utility on the retaining set.

\section{Proofs}
\label{app:proofs}

In this section, we provide the proofs of all the lemma and theorems in this paper. 
\setcounter{theorem}{1}
\begin{lemma}[restated from \cite{xin2022current}]
For any $0<c<1$, the model $\bm\theta^*_\text{LSP}(c)=\text{argmin}_\theta \LSP(\bm\theta; c)$ is Pareto optimal.
\end{lemma}

\begin{proof}[Proof of \Cref{thm:pareto}]
We show that the solution of LSP cannot be a dominated point, and therefore it must be Pareto optimal. Consider a solution $\bm\theta^* = \text{argmin}_\theta \LSP(\bm\theta; c)$, and suppose it is dominated by some $\bm\theta'$, i.e. $L_\F(\bm\theta^*)\leq L_\F(\bm\theta'),L_\R(\bm\theta^*)\geq L_\R(\bm\theta')$ with at least one inequality being strict. This contradicts that $\bm\theta^*$ is minimal as 
$cL_\R(\bm\theta^*)-(1-c)L_\F(\bm\theta^*)>cL_\R(\bm\theta')-(1-c)L_\F(\bm\theta').$
\end{proof}

\begin{theorem}
\label{thm:pareto!}
For any $\{c_t\}$ with $0\leq c_t\leq 1$ that converges as $t\to\infty$, the model $\bm\theta^*(\{c_t\})=\lim_{t\to\infty} \bm\theta_t$ in \eqref{eq:dynamic gradient} is Pareto optimal.
\end{theorem}

\begin{proof}[Proof of \Cref{thm:pareto!}]
Let $c=\lim_t c_t$, then Eq. \eqref{eq:dynamic gradient} gives that $\g_\text{UN}(\bm\theta_t)=c_t\g_\R(\bm\theta_t)-(1-c_t)\g_\F(\bm\theta_t)\to c\g_\R(\bm\theta^*)-(1-c)\g_\F(\bm\theta^*)=\bm{0}$ as $t\to\infty$. Note $\bm\theta^*(\{c_t\})$ is equivalent to the LSP solution $\bm\theta^*_\text{LSP}(c)=\text{argmin}_\theta \LSP(\bm\theta; c)$ as the latter has the same stationary condition, which is Pareto optimal by \Cref{thm:pareto}.
\end{proof}

\begin{lemma}
$\g_\textup{NGDiff}(\g_\R,\g_\F)$ satisfies Eq. \eqref{eq:sanity check} for any $\g_\R\in\mathbb{R}^d$ and $\g_\F\in\mathbb{R}^d$. 
\label{fact:gnorm good}
\end{lemma}

\begin{proof}[Proof of \Cref{fact:gnorm good}]
We firstly show $\g_\R^\top\g_\text{UN}\geq 0$ for $\g_\text{UN}=\g_\textup{NGDiff}$. We write
\begin{equation*}
\begin{split}
\g_\R^\top\g_\textup{NGDiff}&=\g_\R^\top\left(
\frac{\g_\R}{\|\g_\R\|}-\frac{\g_\F}{\|\g_\F\|}\right) \\ & =\|\g_\R\|-\frac{\g_\R^\top\g_\F}{\|\g_\F\|}\geq \|\g_\R\|-\frac{\|\g_\R\|\|\g_\F\|}{\|\g_\F\|} \\ & =0
\end{split}    
\end{equation*}
where the inequality is the Cauchy-Schwarz inequality. Similarly, $\g_\F^\top\g_\text{UN}\leq 0$ easily follows.
\end{proof}

\begin{theorem}
Consider $\bm\theta_{t+1}=\bm\theta_t-\eta \g_\textup{NGDiff}$. 

(1) Unless $\g_\R$ is exactly parallel to $\g_\F$, for any sufficiently small learning rate $\eta$, there exist two constants $\epsilon_{\R,1}=o(\eta), \epsilon_{\F,1}=o(\eta)$ such that
$$L_\R(\bm\theta_{t+1})-L_\R(\bm\theta_{t})< \epsilon_{\R,1};$$ $$L_\F(\bm\theta_{t+1})-L_\F(\bm\theta_{t})> \epsilon_{\F,1}.$$

(2) If additionally $\g_\textup{NGDiff}^\top\H_\R\g_\textup{NGDiff}>0$ and $\g_\textup{NGDiff}^\top\H_\F\g_\textup{NGDiff}>0$, then for any learning rate $0<\eta<\frac{2\g_\R^\top\g_\textup{NGDiff}}{\g_\textup{NGDiff}^\top \H_\R \g_\textup{NGDiff}}$, 
there exist two constants $\epsilon_{\R,2}=o(\eta^2), \epsilon_{\F,2}=o(\eta^2)$ such that 
$$L_\R(\bm\theta_{t+1})-L_\R(\bm\theta_{t})< \epsilon_{\R,2}; L_\F(\bm\theta_{t+1})-L_\F(\bm\theta_{t})> \epsilon_{\F,2}.$$
\label{thm:NGDiff}
\end{theorem}

\begin{proof}[Proof of \Cref{thm:NGDiff}]
% We present $L_\R$ and state that the same analysis follows for $L_\F$. 
Applying \eqref{eq: landscape} with $\g_\textup{NGDiff}$ gives
\begin{align}
&L_\R(\bm\theta_{t+1})-L_\R(\bm\theta_{t})
= -\eta \g_\R^\top\g_\textup{NGDiff}
+\frac{\eta^2}{2}\g_\textup{NGDiff}^\top \H_\R \g_\textup{NGDiff}+o(\eta^2) 
\label{eq:taylor NGD}
\end{align}

For part (1), note that \Cref{fact:gnorm good} gives $\g_\R^\top\g_\text{NGDiff}>0$ unless $\g_\F\parallel\g_\R$. Hence for any $\eta>0$, we have
$$
L_\R(\bm\theta_{t+1})-L_\R(\bm\theta_{t})= -\eta \g_\R^\top\g_\textup{NGDiff}
+o(\eta)<o(\eta)
$$
and similarly for $L_\F$.

For part (2), now that $\g_\textup{NGDiff}^\top \H_\R \g_\textup{NGDiff}>0$, 
% we have the following inequality when $\eta$ is between 0 and $\frac{2\g_\R^\top\g_\textup{NGDiff}}{\g_\textup{NGDiff}^\top \H_\R \g_\textup{NGDiff}}$,
% $$-\eta\g_\R^\top\g_\textup{NGDiff}
% +\frac{\eta^2}{2}\g_\textup{NGDiff}^\top \H_\R \g_\textup{NGDiff}\leq 0$$
% By \Cref{fact:gnorm good}, 
we have
\begin{equation*}
\begin{split}
&-\eta\g_\R^\top\g_\textup{NGDiff}
+\frac{\eta^2}{2}\g_\textup{NGDiff}^\top \H_\R \g_\textup{NGDiff}< 0
 \Longleftrightarrow 
0 < \eta < \frac{2\g_\R^\top\g_\textup{NGDiff}}{\g_\textup{NGDiff}^\top \H_\R \g_\textup{NGDiff}}
\end{split}
\end{equation*}
and similarly
$$-\eta\g_\F^\top\g_\textup{NGDiff}
+\frac{\eta^2}{2}\g_\textup{NGDiff}^\top \H_\F \g_\textup{NGDiff}> 0 
\Longleftarrow
0<\eta$$
We complete the proof by substituting the inequalities into \eqref{eq:taylor NGD}.
\end{proof}

\begin{remark}
\label{rem:cost}
There is a computational overhead to use GeN, as it requires additional forward passes to estimate $\eta_t^*$. Nevertheless, we only update the learning rate every 10 iterations so that the overhead is amortized and thus negligible.
\begin{proof}[Proof of \Cref{rem:cost}]
We extend our original complexity analysis in Remark \ref{rem:GeN lazy} and provide quantitative analysis of computation overheads in FLOPs (floating point operations). Specifically, assume it takes around N FLOPS to perform one forward pass on one example and 2N FLOPS to back-propagate. The basic GDiff requires roughly 6BTN FLOPS to run with batch size B and total number of iterations T, because it needs 1 forward and 1 backward for the retain set and another for the forget set. Our learning-rate-free NGDiff requires about (6BTN+4BTN/10), hence a 6.6\% increase to unlearn the data. We assure that the extra forward passes do not add memory burden, because they are in gradient-free mode (e.g. under \texttt{torch.no\_grad()} mode). All in all, NGDiff is almost as efficient as GDiff and other unlearning methods, as the target unlearning corpus is usually relatively small. This has been empirically observed by our large-scale (up to 7B) model unlearning experiments.
\end{proof}
\end{remark}
% Bibliography entries for the entire Anthology, followed by custom entries
%\bibliography{anthology,custom}
% Custom bibliography entries only

\end{document}